\documentclass{article}
\usepackage{ijcai25}

\usepackage{times}
\usepackage{soul}
\usepackage{url}
\usepackage[hidelinks]{hyperref}
\usepackage[utf8]{inputenc}
\usepackage[small]{caption}
\usepackage{graphicx}
\usepackage{amsmath}
\usepackage{amsthm}
\usepackage{booktabs}
\usepackage{algorithm}
\usepackage{algorithmic}
\usepackage[switch]{lineno}
\usepackage{relsize}

\newtheorem{theorem}{Theorem}

\title{Imitation Learning via Focused Satisficing}

\author{
Rushit N. Shah$^{1,}$\thanks{Equal contribution}
\and
Nikolaos Agadakos$^{1,*}$\and
Synthia Sasulski$^1$\and
Ali Farajzadeh$^1$ \and \\
Sanjiban Choudhury$^2$\And
Brian Ziebart$^1$\\
\affiliations
$^1$Department of Computer Science, University of Illinois Chicago\\
$^2$Department of Computer Science, Cornell University\\
\emails
\{rshah231, nagada2, afaraj5, bziebart\}@uic.edu,
synthiasasulski@gmail.com,
sanjibanc@cornell.edu
}

\usepackage{amssymb}
\usepackage{mathtools}
\usepackage{comment}
\usepackage{multirow}
\usepackage{enumitem}

\usepackage[capitalize,noabbrev]{cleveref}

\theoremstyle{plain}

\newtheorem{lemma}[theorem]{Lemma}
\newtheorem{corollary}[theorem]{Corollary}
\theoremstyle{definition}
\newtheorem{definition}[theorem]{Definition}

\theoremstyle{remark}

\usepackage[textsize=tiny]{todonotes}

\def\trajF{\ifmmode\phi\else$\phi$\fi} 
\def\costF{\ifmmode\text{f}\else f \fi}

\newcommand\blfootnote[1]{%
  \begingroup
  \renewcommand\thefootnote{}\footnote{#1}%
  \addtocounter{footnote}{-1}%
  \endgroup
}

\definecolor{armygreen}{rgb}{0.09,0.45,0.27}
\newcommand\tbest[1]{\textcolor{armygreen}{\textbf{#1}}}

\crefname{section}{§}{§§}

\DeclareMathOperator*{\argmin}{arg\,min}

\begin{document}

\maketitle

\begin{abstract}
Imitation learning often assumes that demonstrations are close to optimal according to some fixed, but unknown, cost function.
However, according to \emph{satisficing theory}, humans often choose \emph{acceptable} behavior based on their personal (and potentially dynamic) levels of \emph{aspiration}, rather than achieving (near-) optimality. For example, a {\tt lunar lander} demonstration that successfully lands without crashing might be acceptable to a novice despite being slow or jerky.
Using a margin-based objective to guide deep reinforcement learning, our {\bf focused satisficing} approach to imitation learning seeks a policy that surpasses the demonstrator's aspiration levels---defined over trajectories or portions of trajectories---on unseen demonstrations \emph{without explicitly learning those aspirations}. We show experimentally that this focuses the policy to imitate the highest quality (portions of) demonstrations better than existing imitation learning methods, 
providing much higher rates of guaranteed acceptability to the demonstrator, and competitive true returns on a range of environments.\blfootnote{Accepted for publication at the 34th International Joint Conference on Artificial Intelligence (IJCAI 2025).}
\end{abstract}

\section{Introduction}
Hand-engineered policies and reinforcement-learned policies from hand-specified cost functions often fail to perform adequately in complicated tasks of interest (e.g., self-driving).
Prevalent imitation learning approaches \cite{osa2018algorithmic} address this issue either by directly mimicking human demonstrations via behavioral cloning \cite{pomerleau1991efficient} or by estimating reward functions that rationalize demonstrator behavior \cite{ng2000algorithms,abbeel2004}
---both under the assumption that the demonstrator is (near) optimal.
The many advantages autonomous systems have over human actors, including faster reaction time \cite{whelan2008effective}, more precise control \cite{ladha2023advantages}, increased rationality, and lossless memory \cite{miller1956magical}, can violate this assumption and lead to potential value misalignment \cite{amodei2016concrete} between demonstrator and imitator.
New perspectives are needed to train more capable imitation learners from less capable demonstrators without supplemental annotations \cite{christiano2017deep,brown2019extrapolating,rafailov2024direct} or assuming some expert-level demonstrations being available \cite{tangkaratt2021robust}.

\begin{figure}
\begin{center}
\includegraphics[width=4.2cm]{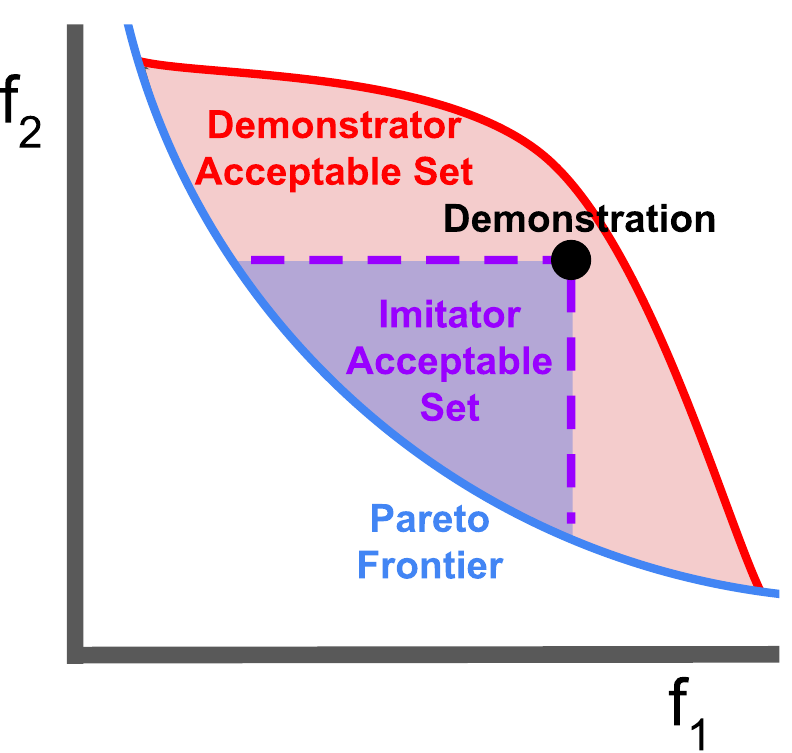}
\includegraphics[width=4.2cm,trim=0 5 30 65, clip]{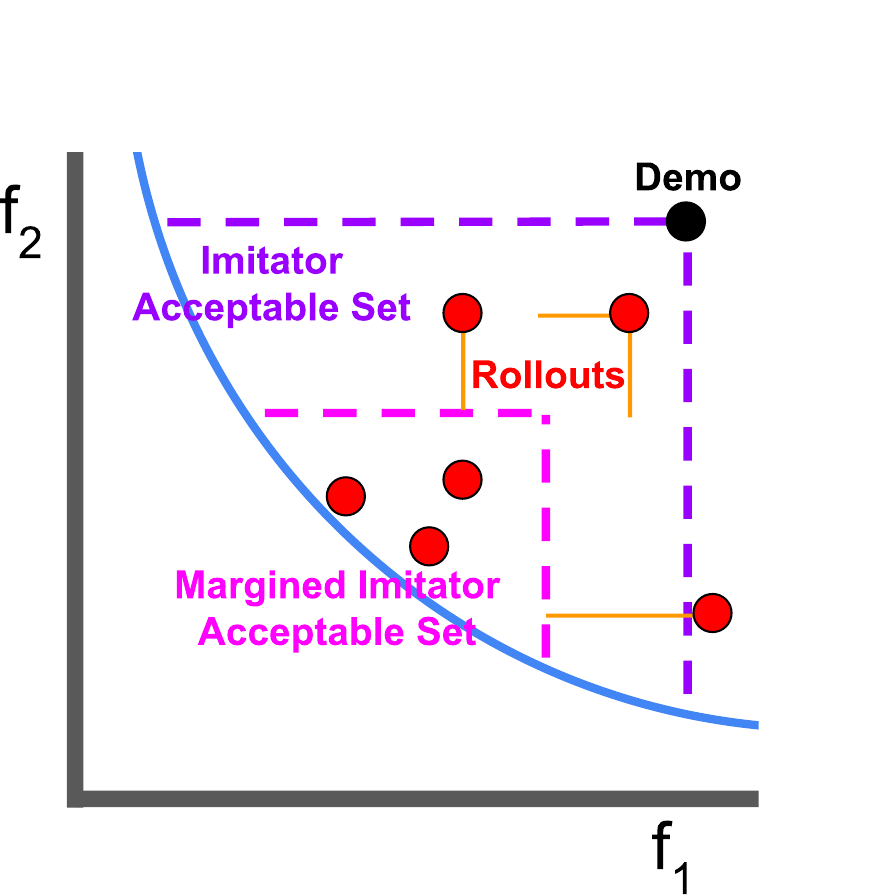}

\end{center}
\caption{Left: Pareto-dominating in the cost function bases (${\textrm f}_1$, ${\textrm f}_2$) of acceptable behavior (purple: \emph{imitator acceptable set}) guarantees the imitator is acceptable to the demonstrator (red: \emph{demonstrator acceptable set}).
Right: The subdominance (orange lines) measures how far imitator trajectory rollouts are from guaranteed acceptance (by a margin).}
\label{fig:satisficing}
\end{figure}

\begin{figure*}[t]
\centering
    \begin{minipage}{0.5\textwidth}
    \includegraphics[width=1.0\columnwidth]{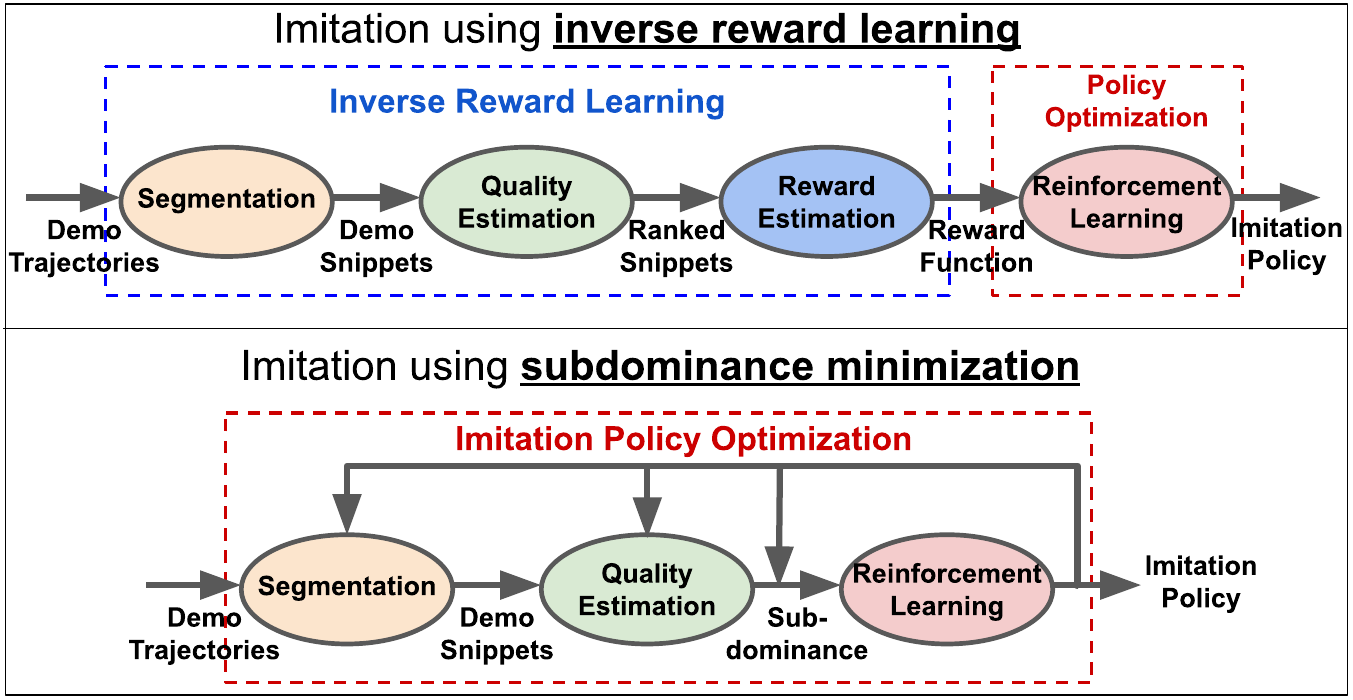}
    \end{minipage}
    \hspace{1mm}
    \begin{minipage}{0.48\textwidth}
    \caption{Existing reward-based
    imitation methods, e.g., TREX \protect\cite{brown2019extrapolating}, seek to outperform demonstrations using a pipeline of engineered components (top) to first segment trajectories into ``snippets,'' and to ultimately estimate a reward function that is then optimized using reinforcement learning. Our approach (bottom) uses the \textbf{subdominance} as the reinforcement learning objective, which is defined by the relative performance of the imitator compared to the demonstrations in each cost feature. This effectively uses feedback from the learned imitator policy to guide additional reinforcement learning without an explicit reward function.}
    \label{fig:model}
    \end{minipage}
\end{figure*}

When faced with challenging decision tasks, \emph{satisficing theory} \cite{simon1956rational} suggests that demonstrators produce behavior that is \emph{acceptable} rather than (near) optimal.
By viewing imitation learning through this lens, we aim for imitator behavior that is similarly \emph{acceptable} to the demonstrator, despite never knowing the demonstrator's precise acceptability criteria (Figure \ref{fig:satisficing}, left)---working instead with an assumed class of cost functions that defines it. 
To pursue this aim, we develop \underline{\textbf{Min}}imally \underline{\textbf{Sub}}dominant \underline{\textbf{F}}ocused \underline{\textbf{I}}mitation (\textbf{MinSubFI}),
which employs the subdominance \cite{ziebart2022towards}, a margin-based measure of \emph{insufficiency} (i.e., the distance from guaranteeing imitator-acceptability by a margin), as a training objective for policy gradient optimization
(Figure \ref{fig:satisficing}, right).
This produces policies that are maximally acceptable rather than reward-maximizing.
Compared to existing inverse reward learning methods \cite{brown2019extrapolating,burchfiel2016distance,wirth2017survey,wu2019imitation,chen2020learning,zhang2021confidence}, which are highly reliant on 
an estimated scalar reward function to guide reinforcement learning (e.g., using the pipeline of engineered components in Figure \ref{fig:model}, top), our approach more directly optimizes the imitator's policy,  enabling it to:

\begin{itemize}[noitemsep,topsep=0pt]
\item Learn context-sensitive policies \underline{without} learning context-sensitive cost functions;
\item Ignore less optimal demonstrations \underline{without} requiring explicit noise modeling;
\item Automatically select and learn from portions of trajectories (i.e., snippets) of high quality; and
\item Provide generalization guarantees for changing acceptability (e.g., due to skill improvement or fatigue).
\end{itemize}

Under the MinSubFI objective, many of the same engineered components of existing approaches (Figure \ref{fig:model}, top) are jointly optimized in a unified manner (Figure \ref{fig:model}, bottom).
We evaluate the benefits of MinSubFI on imitation learning tasks using human and synthetic demonstrations with both engineered and learned cost features.

\section{Satisficing Demonstrations \& Policy Gradient Subdominance Minimization}
\label{seq: minsub_policy_grad}

We now formally recast imitation learning through the lens of \emph{satisficing theory}.
Under this perspective, policies are learned from demonstrations that are \emph{acceptable}, according to an unknown acceptability set, rather than \emph{near-optimal}.
We broadly define this notion of acceptability over trajectories and trajectory ``snippets" (i.e., portions of trajectories), and develop new imitation learning methods that are designed to be performant with respect to the demonstrators' unknown acceptability sets in both theory and practice.

\subsection{Imitation Learning Problem Setting}

We consider the 
imitation learning \cite{osa2018algorithmic} task of producing a policy $\hat{\pi}$ based on demonstrated trajectories of states and actions, $\tilde{\xi} = (\tilde{s}_1, \tilde{a}_1, \tilde{s_2}, \ldots, \tilde{s}_T)$.
Demonstrations are produced from a task-indexed Markov decision process (MDP), $\mathcal{M}=(\mathcal{S},\mathcal{A},\{\tau_i\},C)$, characterized by states $\mathcal{S}$,
actions $\mathcal{A}$, state transition probability distributions $\tau_i : \mathcal{S} \times \mathcal{A} \rightarrow \Delta_{\mathcal{S}}$ (with $\Delta$ representing a probability simplex), and a cost function $C : \mathcal{S} \rightarrow \mathbb{R}_{\geq 0}$. 
The state transition probability distribution is defined for $s=a=\varnothing$ to provide an initial state distribution.
Each state transition probability distribution, $\tau_i$, corresponds to a different task $i$ that shares the same state-action space, but may have different initial states, different absorbing (goal) states, or different dynamics more generally.
We use $\tilde{\xi}_{i,j}$ to denote the j$^{\text{th}}$ demonstration for the i$^{\text{th}}$ task, $\tilde{\Xi}$ to denote the set of all demonstrations, and $\tilde{\Xi}_i$ to denote the set of demonstrations corresponding to task $i$.
The cost/reward function is unavailable to the imitator (providing at most $\mathcal{M}\backslash C$), distinguishing imitation learning from (offline) reinforcement learning \cite{levine2020offline}.

\subsection{Satisficing Perspective of Demonstrations}

According to satisficing theory \cite{simon1956rational}, when faced with challenging decision tasks, humans tend to prioritize behaviors that are acceptable to them rather than striving for optimality.
This implies that demonstrated behavior is selected to be 
\emph{acceptable}, according to some aspirational criteria of the demonstrator, rather than being (near) \emph{optimal}.

\begin{definition} 
Trajectory $\xi$ {\bf satisfices} (or is {\bf acceptable}) for a particular {\bf aspiration}, defined by $({\bf w}, \nu, t, t')$ if and only if it is less costly than the aspirational threshold $\nu$ evaluated using the cost function parameterized by ${\bf w}$: $\text{cost}_{\bf w}(\xi_{t:t'}) < \nu$. It {\bf satisfices} the {\bf aspiration/acceptability set} 
$\Omega = \{({\bf w}, \nu, t, t')\}$, i.e., $\xi \in \text{Satisf}_{\Omega}$,
if and only if $\xi$ satisfices each aspiration in 
$\Omega$. 
\end{definition}

Note that the aspiration set can be context-dependent and vary for each demonstration. 
For example, it may change with the growing experience (or fatigue) of the demonstrator, or based on available side information (e.g., the weather conditions when controlling a vehicle).
Additionally, each aspiration criteria can be defined over a portion (i.e., a ``snippet") $\xi_{t:t'}$ of the full trajectory $\xi_{1:T}$.

Aspiration sets---and their relationships to available contextual information---are generally unknown. Our aim is not to learn them explicitly.  
Instead, we seek a policy that produces trajectories $\xi \sim \pi \times \tau$, with {\bf maximal probability of acceptance}, $P(\xi \in \text{Satisf}_{\tilde{\xi}})$, for $\tilde{\xi}$'s implicit satisfaction set.

A {\bf key question} from this satisficing perspective is: \emph{do existing imitation learners provide acceptability guarantees with respect to (unknown) demonstrator acceptability sets?}

{\bf Behavioral cloning} approaches \cite{pomerleau1991efficient} directly estimate a (stochastic) policy $\pi_\theta : \mathcal{S} \rightarrow \Delta_{\mathcal{A}}$ from demonstrated state-action pairs, $(s_t,a_t)$.
The simplicity of this approach allows the full range of supervised machine learning techniques to be employed to estimate the policy.
For example, generative adversarial imitation learning (GAIL) \cite{gail} employs a discriminator to distinguish between human and automated action choices, and guide policy learning to minimize any differences.
Unfortunately,  behavioral cloning methods cannot outperform the demonstration policy beyond being Bayes optimal for a predictive loss that may not align with the acceptability set cost function(s).
This prevents behavioral cloning methods from providing satisficing guarantees. 

{\bf Inverse reinforcement learning} \cite{kalman1964linear} estimates the cost function $C(s)$ that explains or rationalizes demonstrations (making them near optimal).
A cost function linear in a set of state features, ${\bf f} : \mathcal{S} \rightarrow \mathbb{R}^K$, or state-action features, ${\bf f} : \mathcal{S} \times \mathcal{A} \rightarrow \mathbb{R}^K$ is commonly assumed \cite{ng2000algorithms}.
Under this assumption, {\bf feature matching} \cite{abbeel2004} guarantees the estimated policy $\hat{\pi}$ has expected cost under the demonstrator's unknown fixed cost function weights $\tilde{w} \in \mathbb{R}^K$ equal to the average of the demonstration policies $\pi$ if the expected feature counts match:
\begin{align}
\label{eq:fconstraint}
    &\mathbb{E}_{\substack{\!\!\tau_i \sim \tilde{\Xi},\\ \xi \sim \pi \times \tau_i}} [f_k(\xi)] =\frac{1}{|\tilde{\Xi}|} \sum_{\tilde{\xi}_{i,j} \in \tilde{\Xi}} f_k(\tilde{\xi}_{i,j}),\, \forall k
    \\    &
    \implies \mathbb{E}_{\substack{\!\!\tau_i \sim \tilde{\Xi},\\ \xi \sim \pi_\theta \times\tau_i}} [C_{\hat{w}}(\xi)] = \frac{1}{|\tilde{\Xi}|} \sum_{\tilde{\xi}_{i,j} \in \tilde{\Xi}} C_{\tilde{w}}(\tilde{\xi}_{i,j}), \nonumber 
\end{align}
where $f_k(\xi)\!\triangleq\!\sum_{s_t,a_t \in \xi} f_k(s_t,a_t)$ and
$C_{\hat{w}}(\xi)\!\triangleq\!\sum_{s_t,a_t \in \xi} C_{\hat{w}}(s_t,a_t)$. 
This feature-matching constraint \eqref{eq:fconstraint} can be enforced using a potential term measuring the demonstration $\tilde{\xi}$'s suboptimality relative to induced behavior $\xi$.
Closer to our approach, game-theoretic apprenticeship learning \cite{syed2007game} assumes the sign of the linear cost function's weights are known and produces a policy that is guaranteed to be better in expectation than the demonstration average under worst-case weights. 

Unfortunately, matching the demonstrator's unknown expected rewards (or outperforming on average) only guarantees that the imitator achieves the aspiration level in expectation.
If the demonstrators' aspirations depend on context that is not incorporated in the learned cost function, better levels of aspiration will not be guaranteed.
Thus, inverse reinforcement learning does not provide useful guarantees for per-demonstration satisficing; it is not a discriminative enough policy optimization method.

\subsection{Subdominance Minimization and Satisficing}

The subdominance measures how far trajectory $\xi$ is from Pareto-dominating (i.e., smaller in each cost feature dimension than) a demonstrated trajectory $\tilde{\xi}$ by a margin (Figure \ref{fig:satisficing}, right).
It has been previously employed for inverse optimal control to make the optimal trajectory induced by learned linear cost function weights $\boldsymbol{w}\in\mathbb{R}^{K}_{\geq0}$,
outperform sets of task-specific demonstrations $\{\tilde{\Xi}_i\}$ \cite{ziebart2022towards}:
{\small
\begin{align}
& \min_{\boldsymbol{w}\geq0} \min_{\boldsymbol{\alpha} \geq 0}  \sum_{i=1}^{N} \frac{|\tilde{\Xi}_i|}{|\tilde{\Xi}|}\text{subdom}_{\boldsymbol{\alpha}} (\xi^{\ast}_i({\bf w}), \tilde{\Xi}_i) + \frac{\lambda}{2}||\boldsymbol{\alpha}||, \notag \text{ where:}\\
&  
\text{subdom}_{\boldsymbol{\alpha}}(\xi,\tilde{\Xi}) \!=\! \frac{1}{|\tilde{\Xi}|}\sum_{\tilde{\xi} \in \tilde{\Xi}}
\underbrace{\sum_k \overbrace{\left[\alpha_k(f_k(\xi)-f_k(\tilde{\xi}))+1\right]_{+}}^{\text{(feature $k$) subdom}^k_{\alpha_k}(\xi,\tilde{\xi})}}_{
\text{(aggregated) subdom}_{\boldsymbol{\alpha}}(\xi,\tilde{\xi})}, \label{eq:subdom_agg}
\end{align}}%
with $[x]_{+} \triangleq \max(x,0)$ as the hinge function, and trajectory cost features ${\bf f}:\Xi\rightarrow\mathbb{R}^{K}_{\geq0}$. 
Other variants include defining the subdominance using relative cost features, $\text{relsubdom}^k_{\alpha_k}(\xi,\tilde{\xi}) \triangleq \left[\alpha_k \left(\frac{f_k(\xi)}{f_k(\tilde{\xi})}-1\right)+1\right]_{+}$, and/or aggregating over feature dimensions using maximization, $\text{subdom}_{\boldsymbol{\alpha}}(\xi,\tilde{\xi}) \triangleq \max_k \text{subdom}^k_{\alpha_k}(\xi,\tilde{\xi})$
\cite{ziebart2022towards}.
Like support vector machines \cite{vapnik2000bounds}, only a subset of \emph{support demonstrations}, $\tilde{\Xi}_{i}^{\text{SV}_k}(\xi) \subseteq \tilde{\Xi_i}$, for each task $i$ and feature $k$, actively influence $\boldsymbol{\theta}$:
\begin{align}\tilde{\xi} \in \tilde{\Xi}_{i}^{\text{SV}_k}(\xi) \iff
f_k(\xi) + \frac{1}{\alpha_k} \geq f_k(\tilde{\xi}). 
\label{eq:subdominance}
\end{align}
For notational convenience, when $\xi$ is indexed (e.g., by $(i,j)$ as $\xi_{i,j}$), we denote this resulting support vector set for all demonstrations of task $i$ as $\tilde{\Xi}_{i,j}^{\text{SV}_k}$ for feature $k$.
Unfortunately, optimal control is impractical for many realistic imitation learning problems of interest.
Additionally, it makes the learned cost/reward function (Fig. \ref{fig:model}) a bottleneck that can prevent the imitation policy from better fitting to (or outperforming) demonstrations.

However, subdominance has an important relationship to satisficing (Theorem \ref{thm:satisficing}): if it can be lowered to zero, acceptability of the imitator's behavior is guaranteed under mild cost function assumptions (positive linear functions of monotonic transformations of cost features).

\begin{theorem} 
\label{thm:satisficing} 
A trajectory $\xi$ with zero subdominance with respect to demonstration $\tilde{\xi}$ implies that the demonstration's corresponding aspiration set (for full trajectory aspiration functions/threholds) is satisficed by $\xi$:
$\left(\exists \alpha > {\bf 0}, \text{subdom}_{\alpha}(\xi, \tilde{\xi}) = 0\right) \implies \xi \in \text{Satisf}_{\tilde{\xi}}$.
\end{theorem}


\begin{proof}[Proof of Theorem \ref{thm:satisficing}]
Zero subdominance implies Pareto dominance of the imitator cost feature over the demonstrator cost features, which implies that the imitator is acceptable under any cost functions defining the demonstrator's acceptable set.
\begin{align}
    &\forall \boldsymbol{\alpha} \succ {\bf 0}, 
    \text{subdom}(\xi, \tilde{\xi}) = 0  \implies
    {\bf f}(\xi) \preceq {\bf f}(\tilde{\xi}) \\
    & \qquad \implies \forall  \boldsymbol{\theta} \succeq {\bf 0},  \text{cost}_\theta(\xi) \leq\text{cost}_\theta(\tilde{\xi})\\
    & \qquad \implies \xi \in \text{satisf}_{\tilde{\xi}}
\end{align}
\end{proof}

Note that the additional margin incorporated in the subdominance plays an important role in providing generalization guarantees for the imitator (Theorem \ref{minsubpg_theorem_1}) that do not exist if the imitator simply matches the features of the demonstrator on training examples.


\begin{figure}[thb]
\begin{center}
\includegraphics[width=8.50cm]{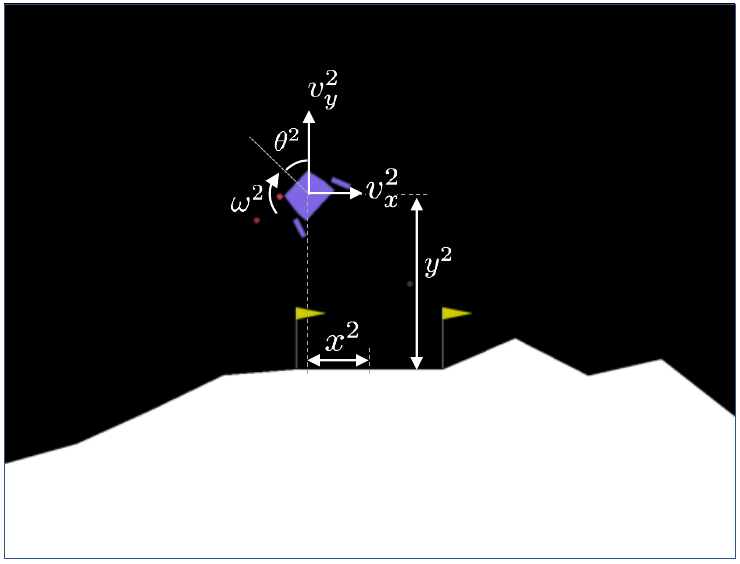}
\end{center}
\caption{Examples of \texttt{lunarlander} cost features, which are computed easily from the environment's observation vector.}
\label{fig:cost_feats_illus}
\end{figure}

{As an illustrative example}, consider two cost features for \texttt{lunar lander} depicted in Figure \ref{fig:cost_feats_illus}: its $x$ offset from the landing pad and its angular velocity $\omega$. An imitator trajectory $\xi$ which lands more precisely (i.e., smaller $x$ offset) and more smoothly (i.e., smaller angular velocity $\omega$) than a demonstration $\tilde{\xi}$, by definition has zero subdominance. Such a trajectory would also be part of the margined imitator acceptable set (Figure \ref{fig:satisficing}, right) and hence satisfices demonstration $\tilde{\xi}$.

Thus, our objective is to better minimize the subdominance by finely optimizing over a more flexible class of policies. To generalize to unseen data, we additionally seek a margin of improvement over the demonstrator, i.e., $\text{subdom}_{\alpha}$, throughout our formulation. With this added margin, the subdominance is a convex function (in trajectory features) that upper bounds the $\text{Satisf}_{\tilde{\xi}}$ non-membership, measuring how far the trajectory is from being guaranteed to satisfy the demonstrator's aspirations by a margin. 

\subsection{Snippet-focused Subdominance}

To enable snippet-level satisficing (in addition to the trajectory-level satisficing guaranteed by Theorem \ref{thm:satisficing}), we define a snippet-focused variant of the subdominance.
We select snippet pairs that maximize subdominance:
{%
\begin{align}
\text{subdom}&_{\boldsymbol{\alpha}}^{\text{snip } \mathcal{S}}(\xi,\tilde{\xi})\!=\!\!\!\!\!\!\max_{(\xi_{\text{sn}}, \tilde{\xi}_{\text{sn}}) \in \mathbb{S}(\xi, \tilde{\xi})} \!\!\!\!\text{subdom}_{\boldsymbol{\alpha}}(\xi_{\text{sn}},\tilde{\xi}_\text{sn}), 
\label{eq:snippet}
\end{align}}%
where $\mathbb{S}$ extracts snippet pairs from the full trajectories. This focuses imitation on high-quality snippets (with high subdominance) even if the larger trajectory they come from is of lower quality (and low or zero subdominance because it is easy for the imitator to outperform as a whole). 
The design of $\mathbb{S}$ provides a great deal of flexibility for defining snippets based on states and/or time steps.

\subsection{Subdominance 
for Stochastic Policies}
\label{sec:subdom_policies}

We next expand the subdominance definition to incorporate stochastic action selection from policy $\pi$:
\begin{align}
\text{subdom}_{\boldsymbol{\alpha}}(\pi,\!\tilde{\Xi})\!=\! \mathbb{E}_{\xi\sim\pi\times \tau}\!\!\left[\text{subdom}_{\boldsymbol{\alpha}}(\xi,\! \tilde{\Xi})\right]
    \!.
\end{align}
\begin{definition}
\label{minsubpg_obj_def}
    The minimally subdominant stochastic policy $\pi_{\boldsymbol{\theta}} : \mathcal{S} \rightarrow \Delta_\mathcal{A}$ minimizes the expected subdominance of the minimum cost trajectory, $\xi^*(\pi_{\boldsymbol{\theta}})$ induced by the weights $\theta$ of policy $\pi$, with respect to the set of demonstration trajectories ${\tilde{\xi}_i}$ using hinge slopes $\boldsymbol{\alpha}$:
    {\small \begin{align}
    \label{eq:minsubpg_obj_eqn}
        \min_{{\boldsymbol{\theta}}} \min_{\boldsymbol{\alpha} \succeq 0} \sum_{\text{task }i} \frac{|\tilde{\Xi}_i|}{|\tilde{\Xi}|}
        &\textup{subdom}_{\boldsymbol{\alpha}} (\pi_{\boldsymbol{\theta}}, \tilde{\Xi}_{i}) 
        \!+\!\frac{\lambda_\alpha}{2}||\boldsymbol{\alpha}||
        \!+\!\frac{\lambda_\theta}{2}||\boldsymbol{\theta}||.
    \end{align}}
\end{definition}

This optimization seeks hinge loss slopes $\boldsymbol{\alpha}$ and a policy $\pi_{\boldsymbol{\theta}}$ that both minimize the subdominance.
Naively approaching this optimization can be problematic, since $\boldsymbol{\alpha}=0$ corresponds to a degenerate local optimum.  However, the optimal $\boldsymbol{\alpha}$ values for a policy achieving at least the average feature counts of the demonstrations are not degenerate.
This suggests bootstrapping from an initial policy estimate when minimizing $\boldsymbol\alpha$ values or restricting $\boldsymbol\alpha$ values above zero.

\subsection{Subdominance Policy Gradient Optimization}
\label{sec:policy_gradient}

{To efficiently optimize the objective outlined in Definition \ref{minsubpg_obj_def}, we consider policy gradient algorithms.
We leverage Theorem \ref{minsubpg_theorem} for the computation of the policy gradient using the trajectory-based subdominance as a reinforcement signal. Corollary \ref{corollary_decompose} provides a per-state decomposition of the subdominance, making a wide range of existing policy gradient methods applicable that assign credit in a temporally consistent manner.}

\begin{theorem}
\label{minsubpg_theorem}
    Policy $\pi_{\boldsymbol{\theta}}$'s subdominance with respect to  demonstration set $\{\tilde{\Xi}_i\}$ has policy gradient:
    {\small
    \begin{align}
& \nabla_{\boldsymbol{\theta}} \sum_i \frac{|\tilde{\Xi}_i|}{|\tilde{\Xi}|} \mathbb{E}_{\xi_i \sim \pi_{\theta} \times \tau_{i}}{\left[ \textup{subdom}_{\boldsymbol{\alpha}}(\xi_i, \tilde{\Xi}_{i}) \right]} \notag 
\\
        = &
        \sum_i\frac{|\tilde{\Xi}_i|}{|\tilde{\Xi}|} \mathbb{E}_{\xi_i \sim \pi_\theta\times\tau_i}{\biggl[  \textup{subdom}_{\boldsymbol{\alpha}}(\xi_i, \tilde{\Xi}_{i}) \sum_{\mathclap{(s,a) \in \xi_i}}\nabla_{\boldsymbol{\theta}} \log\pi_{\boldsymbol{\theta}}{\left(a|s\right)}\biggr]}, \notag
    \end{align}
    }
For a set of single trajectory samples, $\xi_i \sim \pi_\theta \times \tau_i$, for each task $i$, the policy parameters $\boldsymbol{\theta}$ can be (stochastically) updated via gradient descent:
{
        $\boldsymbol{\theta} \leftarrow  \boldsymbol{\theta} +\eta\sum_i
        \sum_{(a_t,s_t) \in \xi_i}G_{t}\nabla_{\boldsymbol{\theta}} \log\pi_{\boldsymbol{\theta}}\left(a_t|s_t\right)
$,   %
    }%
    where $G_{t}$ is any function of the full or future expected subdominance, $\textup{subdom}_{\boldsymbol{\alpha}}(\xi_i, \tilde{\Xi}_{i})$, such as the Q-value, the advantage estimate, or the trajectory return \cite{sutton1999policy}.
\end{theorem}
The proof for Theorem \ref{minsubpg_theorem} is provided in our supplementary material.
{We now present a state-based decomposition of trajectory level subdominance. Subdominance is computed at the final state to determine which features contribute to it, and the contribution of each state-action pair to the total trajectory subdominance is the calculated.}

\begin{corollary}
        \label{corollary_decompose}
        
        The absolute and relative subdominances for a trajectory $\xi$, 
        with respect to a set of demonstrations $\Xi$ can be further expanded as: 
        {\small
        \begin{align*}
        & 
        \textup{subdom}_{\boldsymbol{\alpha}}(\xi, \tilde{\Xi})
        = \!\!
        \sum_{s_t \in \xi,k} \left(
                \!
                \frac{\tilde{C}_{\xi,\tilde{\Xi}}^{k}}{|\xi|} + 
                \tilde{C}_{\xi,\tilde{\Xi}}^{k}\alpha_k f_k(s_t) -
                \frac{\alpha_k \tilde{f}_{k,\xi,\tilde{\Xi}}^{\textup{abs}}}{|\xi||\tilde{\Xi}|}
                \right)
            ; \notag\\ 
        &  
        \textup{relsubdom}_{\boldsymbol{\alpha}}(\xi, \tilde{\Xi}) 
        = \!\!\!
        \sum_{s_t \in \xi,k} \! \left(
                \frac{\tilde{C}_{\xi,\tilde{\Xi}}^{k}(1 \! - \! \alpha_k)}{|\xi|} + \frac{\alpha_k f_k(s_t)\tilde{f}_{k,\xi,\tilde{\Xi}}^{\textup{rel}}}{|\tilde{\Xi}|}\right)
        ,\notag
        \end{align*}}%
        where $\tilde{C}_{\xi,\tilde{\Xi}}^{k}=\frac{|\tilde{\Xi}^{\textup{SV}_k}(\xi)|}{|\tilde{\Xi}|}$, $\tilde{f}_{k,\xi,\tilde{\Xi}}^{\textup{abs}}=\sum\limits_{\tilde{\xi} \in \tilde{\Xi}^{\textup{SV}_k}(\xi)}\sum\limits_{s'_{t} \in \tilde{\xi}} f_k(s'_{t})$, 
        and $\tilde{f}_{k,\xi,\tilde{\Xi}}^{\textup{rel}}=\sum_{\tilde{\xi} \in \tilde{\Xi}^{\textup{SV}_k}(\xi)} \left( \sum_{s'_{t} \in \tilde{\xi}} f_k(s'_{t}) \right)^{-1}$. 


\end{corollary}

This decomposition enables state-of-the-art reinforcement learning algorithms \cite{schulman2017proximal} that assign credit to actions in a causally consistent manner (i.e., only future returns influence an action's updates) 
to be employed. Further flexibility is gained via the choice of policy representation. 


\subsection{Subdominance Policy Gradient Algorithms}
Algorithm \ref{alg:the_alg} outlines our 
approach for optimization.
For each task ($i$), 
a trajectory is rolled out by sampling from the current learned policy (Line 2).
The cost features of the sampled trajectory and the demonstrated trajectory are compared to determine which dimensions the sampled trajectory does not sufficiently outperform the demonstration, and are thus support vectors (Line 4).
Here, the $\alpha$ values defining margin slopes (Eq. \eqref{eq:subdominance}) can either be optimized numerically (e.g., using stochastic gradient descent) or analytically \cite{memarrast2023superhuman}.
A policy update is then employed to reduce the 
subdominance (Line 7).

\begin{algorithm}[H]
  \caption{Online subdominance policy gradient}
  
  \label{alg:the_alg}
  \begin{algorithmic}[1]
    \WHILE{ $\boldsymbol{\theta}$ not converged }
    \STATE Sample a set of $M$ trajectories $\Xi_{i}=\{\xi_i^{(m)}\}_{m=0}^{M}$ from policy $\pi_{\boldsymbol{\theta}} \times \tau_i$ for each task $i$ \label{op1}
    \FOR{{\bf each } $\xi_i^{(m)}\in\Xi_{i}$}
        \STATE Find support vectors $\tilde{\Xi}_{i,m}^{\text{SV}_{k}}\,\text{(and }\boldsymbol{\alpha}\text{) given}\, \xi_i^{(m)}$
        \STATE Compute loss $\mathcal{L}(\xi_i^{(m)})=\text{subdom}_{\boldsymbol{\alpha}}(\xi_i^{(m)},\tilde{\Xi}_i)$
    \ENDFOR

    \STATE Update $\boldsymbol{\theta}$ via policy gradient update rule on $\mathcal{L}(\xi_i^{(m)})$

    \ENDWHILE
  \end{algorithmic}
 \end{algorithm}

{For snippet-based optimization }(Eq. \eqref{eq:snippet}), the snippet extractor $\mathbb{S}$ produces snippet pairs as support vector candidates.
This can uncover supporting snippets from high-quality portions of trajectories that are lower quality overall (and not supporting trajectories). 
The highest subdominance snippets are then used to compute subdominance losses (Line 5) and to perform policy gradient updates (Line 7).


Algorithm \ref{alg:the_alg_snip_opt} describes a practical approach for snippet-based optimization. We consider the snippet generator, $\mathbb{S}$, that produces snippets of various lengths starting from states that coincide between the rollout and the demonstration.
Unfortunately, demonstrations and rollouts may share very few states (apart from the initial state) in practice. 
To more effectively uncover supporting snippets, we roll out trajectories from randomly-chosen states along the demonstrator trajectory, and compare these to snippets from the demonstration that begin from that state.

\begin{algorithm}[h]
  \caption{Snippet-based subdominance policy gradient}
  
  \label{alg:the_alg_snip_opt}
  \begin{algorithmic}[1]\small
    \WHILE{ $\boldsymbol{\theta}$ not converged }
    \STATE Sample demonstration $\tilde{\xi}_j$ from demonstration set $\tilde{\Xi}$ \label{snip_op1}
    \STATE Sample state $s^{(j)}_{t}\sim\tilde{\xi}_j$ such that $0<t<|\tilde{\xi}_j|$\label{snip_op2}
    \STATE Set $s_0\gets s^{(j)}_{t}$ \label{snip_op3}
    \STATE Sample imitator trajectory $\xi\sim\pi_{\boldsymbol{\theta}}(\cdot|s_0)$
    \STATE Find largest support vector snippets pair(s)
$(\xi_{\text{snip}},\tilde{\xi}_{\text{snip}})$ (and $\boldsymbol{\alpha}$) from snippet pair candidates $\mathbb{S}(\xi,\tilde{\xi}_{t:T})$
    \STATE Compute loss $\mathcal{L}(\xi_\text{snip})=\text{subdom}_{\boldsymbol{\alpha}}(\xi_\text{snip},\tilde{\xi}_\text{snip})$

    \STATE Update $\boldsymbol{\theta}$ via any policy gradient update rule on $\mathcal{L}(\xi_\text{snip})$
    \ENDWHILE
  \end{algorithmic}
 \end{algorithm}


When deploying or simulating a policy is expensive,\textbf{ offline policy gradient methods} that are based entirely on the set of demonstrated trajectories can instead be employed. 

\begin{corollary}
\label{coroll_offpolicy}
Offline policy gradient (MinSubFI\textsubscript{OFF}) employs importance weighting to estimate the gradient for online subdominance minimization from available demonstrations:
: {\small
    \begin{align}
    \label{eq:offline_theta_update}
        \boldsymbol{\theta}\!\gets\!\boldsymbol{\theta}\!+\! \eta\sum_{\mathclap{i,\tilde{\xi}_{i,j} \in \tilde{\Xi}_i}}
\tilde{r}_{\theta,\tilde{\pi}}^{(i,j)}\textup{subdom}_{\boldsymbol{\alpha}}(\tilde{\xi}_{i,j}, \tilde{\Xi}_{i})
        \sum_{\mathclap{(s,a) \in \tilde{\xi}_{i,j}}} \nabla_{\boldsymbol{\theta}} \!\log\pi_{\boldsymbol{\theta}}\!\left(a|s\right),
    \end{align}}
    where $\tilde{r}_{\theta,\tilde{\pi}}^{(i,j)}=\frac{\pi_{\boldsymbol{\theta}}(\tilde{\xi}_{i,j})}{\tilde{\pi}(\tilde{\xi}_{i,j})}$ is the importance ratio, and $\tilde{\pi}$ is an estimate of the demonstrator's policy.
\end{corollary}

The offline policy gradient method (Corollary \ref{coroll_offpolicy}) is outlined in Algorithm \ref{alg:the_alg_offpolicy} below. 

\begin{algorithm}[h!]\small
  \caption{Offline, joint stochastic optimization}
  \label{alg:the_alg_offpolicy}
  \begin{algorithmic}[1]          
    \STATE Estimate $\tilde{\pi}_{\text{BC}}$ using behavior cloning on demonstrations $\tilde{\Xi}$  
    \WHILE{ $\boldsymbol{\theta}$ not converged }
    
        \FOR{{\bf each } $\tilde{\xi}_{i,j} \in \tilde{\Xi}_i$}        
             \STATE Find support vectors $
                     \tilde{\Xi}_{i,j}^{\text{SV}_{k}}(
                     \alpha_k)\, \text{given}\, \tilde{\xi}_{i,j}
                     $
            \FOR{{\bf each } $k$}

                        \STATE $\alpha_k\!\gets\!\alpha_k \exp\bigl\{-{\eta}_{t}'
                        \tilde{r}_{\theta,\tilde{\pi}_\text{BC}}^{(i,j)} 
                        \sum\limits_{\mathclap{\tilde{\xi}' \in  \tilde{\Xi}_{i,j}^{\text{SV}_{k}}
                        }}
                        \bigl( f_{k}(\tilde{\xi}_{i,j})\!-\!f_{k}(\tilde{\xi}')\bigr)\!
                        +\!\lambda|\tilde{\Xi}|\alpha_k\bigr\}                        
                        $
            \ENDFOR
            \STATE Update $\boldsymbol{\theta}$ according to Equation \eqref{eq:offline_theta_update}.
        \ENDFOR
    \ENDWHILE
  \end{algorithmic}
 \end{algorithm}


\subsection{Generalization Bound Analysis}
\label{sec:generalization}

We now define the notion of a $\gamma$--{\bf satisficing} stochastic policies and present a generalization bound.

\begin{definition}
\label{gmma_suprhumn}
    A policy is considered $\gamma$--{\bf satisficing} (or $\gamma$--{\bf acceptable}) for cost features ${\bf f}$ and distribution of demonstrated trajectories $P(\tilde{\xi})$, if its trajectories $\xi$ drawn from policy $\pi$ satisfies with probability at least $\gamma$:
        $P(\xi \in \Omega_{\tilde{\xi}}) \geq \gamma$.
\end{definition}

A snippet-based extension of this definition considers $\xi \in \Omega_{\tilde{\xi}}$ if and only if the subdominance is zero for all max-min snippet pairs (Eq. \eqref{eq:snippet}).  

\begin{theorem}
\label{minsubpg_theorem_1}
    The policy minimizing the absolute or relative $\textup{subdom}_{\boldsymbol{\alpha}} \left( \xi^* ( \pi_{\boldsymbol{\theta}} ) , \tilde{\xi}_i \right)$ (N iid demonstrations) with realizable features that are convex sets has the support vector set $\left\{ \tilde{\Xi}_{\text{SV}_k}(\xi^*(\pi_{\boldsymbol{\theta}}),\alpha_k) \right\}$ and is on average $\gamma$--{\bf satisficing} on the population distribution with: $\gamma = 1 - \frac{1}{N} \Big\lvert \bigcup\limits_{k=1}^{K} \tilde{\Xi}_{\text{SV}_k}(\xi^*(\pi_{\boldsymbol{\theta}}),\alpha_k)  \Big\rvert$.
\end{theorem}

This bound motivates subdominance minimization for producing demonstrator-acceptable behavior.

\subsection{Learning a Cost Feature Representation}
Though shaping the imitator's behavior from demonstrations is much less dependent on a highly-expressive cost model/features under our approach,
hand-engineering features can still be a significant burden in many domains.
To mitigate this, we propose to learn a set of cost features $\textbf{f}_{\psi}$ from pairwise preferences over demonstrations.

\begin{definition}
    \label{thrm_repr_rank}
    Given pairwise preferences over demonstrations $\tilde{\mathcal{D}}=\{\tilde{\xi}_i\prec\tilde{\xi}_j | \tilde{\xi}_i,\tilde{\xi}_j\in \Tilde{\Xi}\}$, and a sufficiently-rich function class $\mathcal{F}$, a preference-preserving (latent) representation $\textbf{f}_{\psi}:\mathcal{S}\rightarrow\mathbb{R}^{K'}_{\geq0}$ (of dimensionality  $K'$) can be learned by minimizing:
$\argmin_{\textbf{f}_{\psi}\in\mathcal{F}}\mathbb{E}_{(\tilde{\xi}_i\prec\tilde{\xi}_j)\sim\tilde{\mathcal{D}}}\left[-\log
    \frac{e^{c_{i,j}}}{e^{c_{i,j}} + e^{c_{j,i}}}
    \right]$,
    where $c_{i,j} =\text{subdom}_{\alpha}(\textbf{f}_{\psi}(\tilde{\xi}_i),\textbf{f}_{\psi}(\tilde{\xi}_j))$.
\end{definition}

Given the learned feature representation, a $\gamma$-satisficing policy can be learned via subdominance minimization in Algorithm \ref{alg:the_alg}. This differs from the formulation of TREX in two 
key aspects. First, under the exponential preference model \cite{bradley1952rank,christiano2017deep,brown2019extrapolating,brown2020better}, we employ subdominance between pairs of demonstrations as a loss function, rather than a linear cost function. The second difference emerges from choosing subdominance as the loss function: our formulation permits learning latent representations of \emph{any} dimensionality, rather than just a scalar cost signal; such a vector representation allows us to recover multiple, competing objectives from preferences, rather than arbitrarily extrapolating over a scalar reward signal.

\section{Experiments}
\subsection{Demonstrations}
We conduct experiments using a mix of simple, classic control environments ({\tt cartpole}, {\tt lunarlander}) and complex robotics environments (Mujoco {\tt hopper}, {\tt halfcheetah}, {\tt walker}) from OpenAI Gym \cite{brockman2016gym}. For each environment, we obtain 100 demonstrations from 
a suboptimal policy learned using PPO. This ensures that the majority of the resulting demonstrations are suboptimal and noisy. 
Human demonstrations for the {\tt lunarlander} used in Section \ref{sec:lander_human} are collected from non-expert, human players using the joysticks on an XBox 360 video game controller. Demonstration return statistics for environment-specific demonstration sets of varying quality are provided in Table \ref{tab:demo_moments}.

\begin{table}[h]
    \centering
    \caption{True return statistics of demonstration sets for each environment (100 demonstrations each).}
    \begin{tabular}{rrrrrr}
         \toprule
         Demo Type & Environment & Min & Mean & Max \\
         \midrule
           \multirow{5}{*}{synthetic} & {\tt cartpole}    & $10$    &  $76$   &  $194$ \\
           & {\tt lunarlander} &-196 & 112 & 284\\
           & {\tt hopper}      & $6$     &  $939$  &  $3441$\\
           & {\tt halfcheetah} & $-83$   &  $680$  &  $1483$\\
           & {\tt walker} & $18$   &  $968$  &  $4293$\\
         \midrule
           \multirow{1}{*}{human} & {\tt lunarlander}    & $-419$    &  $173$   &  $303$ \\
    \bottomrule
    \end{tabular}
    \label{tab:demo_moments}
    \end{table}

Cost features are environment-specific as follows:
\begin{itemize}
\item {\tt cartpole}: $(\text{cart position})^2$, $(\text{cart velocity})^2$, and $(\text{pole angle})^2$, $(\text{pole angular velocity})^2$;
\item {\tt lunarlander}: $(\text{x-position})^2$, $(y\text{-position})^2$, $(x\text{-velocity})^2$, $(y\text{-velocity})^2$, $(\text{angle})^2$, $(\text{angular velocity})^2$, and control costs;
\item {\tt hopper}: inverse $x$-velocity, inverse $z$-velocity, inverse $z$-position, inverse torso angle, and control cost;
\item {\tt halfcheetah}: three variants of inverse $x$-velocity, and control cost; and
\item {\tt walker}: inverse $x$-velocity, inverse $z$-position, and control cost.
\end{itemize}

\subsection{Baseline Methods}\label{sec:baselines}
We employ behavior cloning (BC), generative adversarial imitation learning (GAIL) \cite{ho2016generative}, and adversarial inverse reinforcement learning (AIRL) \cite{fu2018learning} as classical imitation baselines. From extrapolative (better-than-demonstrated) imitation approaches, we compare our approach against T-REX \cite{brown2019extrapolating} rather than its more recent extension, D-REX \cite{brown2020better}, since the latter only adds a new method for automatically generating ranked, synthetic demonstrations, while still retaining the core formulation and loss function introduced in T-REX. As a result, we are better able to examine fundamental differences between learning from ranked demonstrations and learning by subdominance minimization. To provide a more comparable baseline, we learn the TREX cost function $C$ as a linear combination of cost features $\mathbf{f}$ and cost function weights $\hat{w}$, rather than as a function mapping from the observation vector $\boldsymbol{\phi}$ to cost $C$ (abbreviated TREX\textsubscript{CF}); this can be thought of as replacing the penultimate layer of T-REX's cost network with a known, predefined state cost representation. We also train an unaltered version of T-REX (abbreviated TREX) on our demonstrations and provide these results for reference.

We implement the policy optimization of MinSubFI using Stable Baselines3 \cite{stable-baselines3}.  Across all experiments, all baseline methods use the same base policy model paired with Stable Baseline3's implementation of the PPO algorithm \cite{schulman2017proximal}. 
The experiments are not based on extensive hyperparameter tuning; rather, all policy networks use nearly the same hyperparameters (Table \ref{tab:hyperparams}).

\begin{table}[htbp]
  \caption{Values of PPO hyperparameters for each environment.}
  \label{tab:hyperparams}
  \centering
  \scalebox{0.7}{
  \begin{tabular}{rrrrrrrr}
    \toprule
    & learning & entropy & mini-  & & & clip & total \\
    Environment & rate & coeff. & batch & horizon & epochs & range & steps\\
    \midrule
    {\tt cartpole}    & $1\mathrm{e}{-4}$ & 0 & $512$ & $2048$ & $10$ & $0.2$ & $2\mathrm{e}6$  \\
    {\tt lunarlander} & $1\mathrm{e}{-4}$ & $1\mathrm{e}{-6}$ & $2048$ & $2048$& $10$& $0.2$ & $2\mathrm{e}6$ \\
    {\tt hopper}      & $9.8\mathrm{e}{-5}$ & $1\mathrm{e}{-2}$& $512$& $2048$& $5$& $0.2$ & $5\mathrm{e}{6}$ \\
    {\tt halfcheetah} & $9.8\mathrm{e}{-5}$ & $1\mathrm{e}{-4}$ & $256$& $2048$& $5$& $0.2$ & $5\mathrm{e}{6}$ \\
    {\tt walker}      & $2\mathrm{e}{-5}$ & $6\mathrm{e}{-4}$ & $32$& $512$& $20$& $0.1$ & $1\mathrm{e}{6}$ \\
    \midrule
  \end{tabular}}
\end{table}

\subsection{Training and Bootstrapping}

Using Algorithm \ref{alg:the_alg} and analytically computed $\alpha$ values in step 4, we initialize our Online MinSubFI training with a policy that is pretrained via Offline MinSubFI (Corollary \ref{eq:offline_theta_update}); we motivate this choice via an ablation study with different policy initializations in the extended version of this paper.
For all of our experiments throughout the paper, we employ a quadratic expansion of the original cost features as the vectorized outer product of the original cost feature vector, ${\bf f}_{\text{expanded}}=\text{vec}({\bf f}\cdot{\bf f}^{\top})$.


\begin{table*}[ht!]
  \caption{Relative $\gamma$-satisficing values of different versions of MinSubFI on the basic cost features (values greater than 1 are formatted in \textbf{bold} and the best of each environment is additionally colored \textcolor{armygreen}{green}).}
  \label{tab:accept}
  \centering
  \scalebox{0.8}{
    \begin{tabular}{
    rrrrrrrrrrr}
    \toprule
     & \multicolumn{5}{c}{Baselines} & \multicolumn{5}{c}{Ours} \\
    \cmidrule(lr){2-6}\cmidrule(lr){7-11}
    Environment & BC &TREX & TREX\textsubscript{CF}& AIRL & GAIL & MinSubFI\textsubscript{OFF}& MinSubFI\textsubscript{ON}& MinSubFI\textsubscript{SNIP}& MinSubFI\textsubscript{SNIP*}& MinSubFI\textsubscript{LCF}\\

    \midrule
    {\tt cartpole} & 
    0.19 & 0.04 & 0.00 & 0.09 & \textbf{2.24}
    & \textbf{2.62} & \textbf{2.64} & \tbest{2.68} & \textbf{2.59} & \textbf{2.04} \\
    {\tt lunarlander} 
    & 0.00 & 0.00&  0.00 & 0.02 & 0.00
    & 0.49 & \textbf{1.03} & \textbf{1.16} & \textbf{1.49} & \tbest{2.24} \\
    {\tt hopper}  
    & 0.00 & 0.00& 0.00& 0.02 & \tbest{6.40} 
    & 0.86 & \textbf{1.63} & \textbf{1.29} & \textbf{1.38} & \textbf{1.46} \\
    {\tt halfcheetah}
    & 0.00 & 0.00 & 0.00& \textbf{1.12} & \tbest{3.99}
    &\textbf{1.93} & \textbf{1.93} & \textbf{1.85} & \textbf{1.77} & \textbf{1.74} \\
    {\tt walker2d} 
    & \tbest{8.73}& 0.00 & 0.00& 0.00 & 0.00 
    & \textbf{2.15}& \textbf{1.44} & \textbf{1.46} & \textbf{1.54} & \textbf{1.86} \\

    \bottomrule
    
  \end{tabular}}
\end{table*}

\begin{table*}[tbh]
  \caption{Mean (and standard deviation) of the true episode returns of the \textbf{held out demonstrations} and trajectories sampled from different imitation learning methods' learned policies for all environments using synthetic demonstrations and for {\tt lunarlander} with human demonstrations (bottom row).}
  \label{tab:true}
  \centering
  \small
  \resizebox{\textwidth}{!}{
    \begin{tabular}{rr@{\hskip.1cm}l@{\hskip0cm}r@{\hskip.1cm}l@{\hskip.2cm}r@{\hskip.1cm}l@{\hskip.2cm}r@{\hskip.1cm}l@{\hskip.2cm}r@{\hskip.1cm}l@{\hskip.2cm}r@{\hskip.1cm}l@{\hskip.2cm}r@{\hskip.1cm}l@{\hskip.2cm}r@{\hskip.1cm}l@{\hskip0.1cm}r@{\hskip.1cm}lr@{\hskip.1cm}l@{\hskip.1cm}r@{\hskip.1cm}ll@{\hskip.1cm}r@{\hskip.1cm}l@{\hskip.1cm}rll@{\hskip.1cm}rl@{\hskip.1cm}r}
    \toprule
    & \multicolumn{2}{c}{Demon-} & \multicolumn{10}{c}{Baselines} & \multicolumn{10}{c}{Ours} &  \\
    \cmidrule(lr){4-13}\cmidrule(lr){14-24}
    Environment & \multicolumn{2}{c}{strations} & 
    \multicolumn{2}{c}{BC} &
    \multicolumn{2}{c}{TREX}&
    \multicolumn{2}{c}{TREX\textsubscript{CF}} & \multicolumn{2}{c}{AIRL} & \multicolumn{2}{c}{GAIL} & \multicolumn{2}{@{\hskip.1cm}c@{\hskip0cm}}{\small MinSubFI\textsubscript{OFF}}& \multicolumn{2}{@{\hskip.1cm}c@{\hskip0cm}}{\small MinSubFI\textsubscript{ON}}& \multicolumn{2}{@{\hskip.1cm}c@{\hskip0cm}}{\small MinSubFI\textsubscript{SNIP}}& \multicolumn{2}{@{\hskip.1cm}c@{\hskip0cm}}{\small MinSubFI\textsubscript{SNIP*}}& \multicolumn{2}{@{\hskip.1cm}c@{\hskip0cm}}{\small MinSubFI\textsubscript{LCF}}  \\
    \midrule
    {\tt cartpole}    & 116 & (74) & 70 & (37)& 199 & (0.1)& 199 & (0.1) & 15 & (4) & \textbf{200 }& \textbf{(0.0)} &
        \textbf{200}  & \textbf{(0.1)} & 198 & (2) & 199 & (0.9) & 197 & (2) & \textbf{200} & \textbf{(0.0)} \\
    {\tt lunarlander} & 113 & (132) & 164 &  (27) & -171 & (3)& {195} & (7) & -416 & (30) & 256 & (9) & 
        \textbf{268}  &\textbf{(0.5)}  & \textbf{268} & \textbf{(0.6)} & \textbf{268} & \textbf{(0.6)} & 266 & (1) & -269 & (153) \\
    {\tt hopper} & 858 & (884) & 671 & (80) & 1335& (15)& \textbf{2657} &\textbf{(28)} & 11 & (4) &601& (30) &
        570 &(33)  & 1470 & (149) & 1070 & (166) & 849 & (125) & 1373 & (305) \\
     {\tt halfcheetah} & 686 & (584) & 1283 & (53) &  1017 & (7)& 1535 &(49) & 768 & (47) &1595&(4) & 
        \textbf{1626} &\textbf{(10)}  & 1591 & (8) & 1591 & (14) & 1576 & (10) & 1463 & (86) \\
    {\tt walker2d}    & 891 & (1141) & 526 & (99) &  20 & (0.0)& 90 &(5) & -3 & (0.1) & 489& (82) & 
        1461 &(449) & 1688 & (479) & 1861 & (481) & 1446 & (402) & \textbf{2592} & \textbf{(182)} \\
        \midrule
            {\tt lunarlander}$_{\text{human}}\!\!\!$ & 173 & (118) & -75 &  (83) & -569 & (114)& {-310} & (86) & -197 & (71) & -254 & (5) & 
        {-25}  &{(21)}  & \textbf{193} & \textbf{(5)} & 6 & (10) & -76 & (11) & -487 & (267) \\
    \bottomrule
  \end{tabular}
}
\end{table*}

We train two snippet-based MinSubFI models: one using fixed snippet lengths and alignments (MinSubFI\textsubscript{SNIP}) and one with alignments that are selected through optimization (MinSubFI\textsubscript{SNIP*}).
We additionally train  an online subdominance minimizer with a \underline{l}earned
\underline{c}ost \underline{f}eature space (MinSubFI\textsubscript{LCF}) of $K'=3$ dimensions via Definition \ref{thrm_repr_rank}. 
We use a multi-layer perceptron network with two hidden layers of width $8$ as our cost feature architecture. 
In contrast with TREX, which employs four different levels of preference (or ranks) to categorize demonstration quality, we consider two preference levels (i.e., \emph{acceptable} and \emph{not acceptable}).

\subsection{Demonstrator Acceptability Analysis} 

 In Table \ref{tab:accept}, we evaluate the rate that the imitator satisfices demonstrations (Definition \ref{gmma_suprhumn}), guaranteeing demonstrator satisfaction, relative to the rate that a randomly chosen demonstration satisfices other demonstrations, $$P(\xi \in \Omega_{\tilde{\xi}})/P(\tilde{\xi'} \in \Omega_{\tilde{\xi}}),$$
 using trajectory-level cost features.
Imitation learning methods designed to minimize a predictive loss (BC) or a learned cost function (TREX) produce trajectories with very different cost features than those of the demonstrations,
leading to small values in this analysis (with a few exceptions, e.g., BC on {\tt walker2d}).
More specifically, aggressively optimizing an estimated cost function using reinforcement learning often focuses too narrowly on minimizing one or a small subset of cost features at the expense of ignoring one or more other features, allowing them to take unacceptable values.
For example, though TREX produces {\tt cartpole} policies keeping the pole upright (near optimally), it does so with much larger amounts of horizontal motion than demonstrations exhibit, making it potentially unacceptable to the demonstrator. GAIL, which employs a discriminator to help make demonstrator and imitator behavior indistinguishable, achieves 
high acceptability rates on some environments.
However, it does so inconsistently, with no relative satisficing on {\tt lunarlander} and {\tt walker2d}. 

In contrast, since MinSubFI minimizes an upper bound on the imitator's satisficing value, it consistently guarantees demonstrator acceptability much more frequently. We also find that online subdominance minimization tends to provide more frequent acceptability guarantees than the offline variant. Additionally, snippet-optimized subdominance minimization frequently provides 
 the large rates of acceptability guarantees in different environments. This is despite the fact that snippet optimization seeks to provide snippet-level demonstrator acceptability, while Table \ref{tab:accept} measures trajectory-level acceptability, indicating its general benefit in guiding policy optimization. 
 
 In addition, though MinSubFI\textsubscript{LCF} learns its own space of cost features, it still provides large rates of guaranteed demonstrator acceptance in the original, provided cost feature space for most environments. This provides some evidence that knowing the demonstrator's cost feature space is unnecessary for providing demonstrator-acceptable behavior, even though it may not be possible to formally guarantee demonstrator acceptance in such settings.

\begin{figure*}[tbh!]
    \centering
    \includegraphics[width=\textwidth]{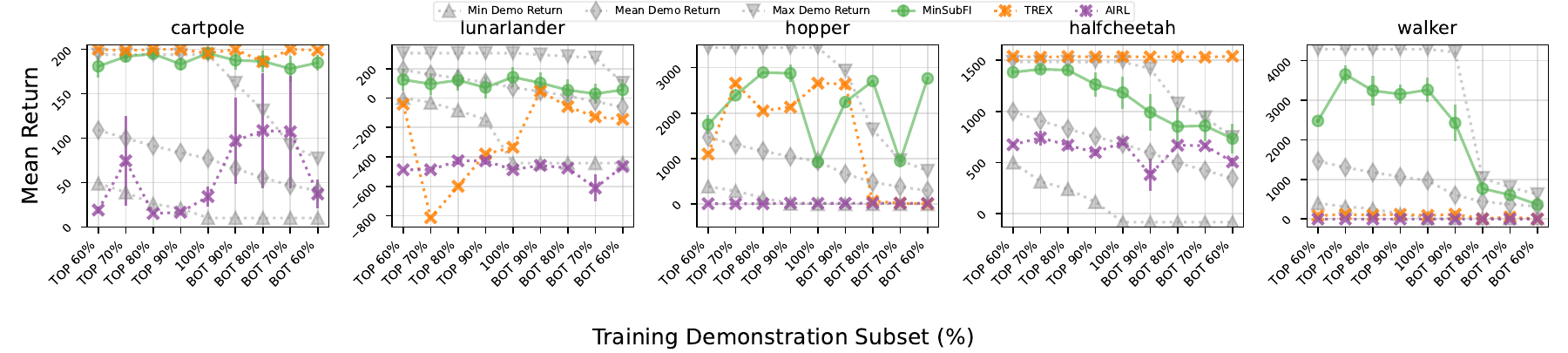}  
\caption{Mean true returns of 100 trajectories rolled out from the learned policies and the minimum, average, and maximum reward of the training set trajectories . Each policy was trained on a subset of demonstrations obtained by removing the best or worst $10\%$, $20\%$, $30\%$, or $40\%$ of the demonstrations. 
Compared to T-REX (orange) and AIRL (purple), the performance of MinSubFI (green) is more robust to increases in the proportion of suboptimal demonstrations in the dataset.}
    \label{fig:demo_quality_vs_performance}
\end{figure*}

\subsection{True Returns Using Full Demonstration Set}
Unlike ``real-world'' imitation learning tasks, the true returns used to construct the demonstrator's policy are known in our experiments.
Though MinSubFI seeks to achieve demonstrator acceptability for all cost functions defined by its cost features (Table \ref{tab:accept}), it should also provide improvements over demonstrations in terms of true return when the true return can be (approximately) defined by the cost features.
We note that having a fixed true return function for all demonstrations is a strong assumption from the perspective of our formulation of satisficing theory, which allows the acceptability set to vary with each demonstration. 
In Table \ref{tab:true}, we evaluate the true returns of the demonstrations and each of the imitation learning methods averaged over three random seeds. 

We find that Behavioral Cloning (BC) and AIRL often underperforms relative to the demonstrations (except for {\tt halfcheetah}), while the relative performance for TREX and GAIL is more mixed.  
We note that GAIL's higher relative satisficing performance in Table \ref{tab:accept} does not translate into higher true returns ((.g., on {\tt hopper}).
In contrast, the various forms of MinSubFI tend to consistently outperform the demonstrations with only a few exceptions (e.g., MinSubFI\textsubscript{OFF} on {\tt hopper}). 
In terms of the numerical true returns, different variants of MinSubFI provide the highest returns except for {\tt hopper}, in which TREX\textsubscript{CF} provides the largest returns.
Interestingly, while TREX benefits from using the cost features as the basis for its cost function estimate, cost function learning provides significant improvements for MinSubFI\textsubscript{LCF} in {\tt walker2d} and {\tt cartpole}.

\subsection{Demonstration Quality and Performance}
Demonstrated behavior is often noisy and suboptimal, making learning from such data 
a desirable capability. In this section, we control the quality of demonstrations used for imitation. We sort all demonstrations by their total (true) return and then choose a subset by retaining the best or worst $90\%$, $80\%$, $70\%$, or $60\%$ of the original set. We use this demonstration subset to train T-REX and Online MinSubFI. The performance is shown in Figure \ref{fig:demo_quality_vs_performance}.

For the simple {\tt cartpole} environment, both TREX and MinSubFI are able to continue outperforming the best demonstrations even when they become worse in quality.
For the remaining environments, except for {\tt halfcheetah} in which TREX performs exceptionally well, MinSubFI tends to provide comparatively better true returns than the baselines as the quality of demonstrations becomes worse.

\subsection{Performance with Human Demonstrations}
\label{sec:lander_human}
Human-provided demonstrations often exhibit multi-modal patterns and irregular noise distributions, making them harder to learn from. We explore the behavior of our method trained on human-provided demonstrations, using the continuous version of Lunarlander and compare it against imitation baselines. The results, averaged over three seeds, are shown in the last row of Table \ref{tab:true}. MinSubFI\textsubscript{ON} clearly outperforms all baselines, which appear unable to learn from human demonstrations in this setting.

\section{Conclusions and Future Work}\label{conclusions}
In this paper, we reframed imitation learning using satisficing theory to 
develop MinSubFI, a policy gradient approach for seeking to make the imitator's behavior acceptable to the demonstrator by directly minimizing policy subdominance---based on entire trajectories or selectively chosen snippets. We present both variants for online and offline learning, and show how offline bootstrapping results in significant simulator sample efficiency. Further, we present a feature presentation learning method using offline subdominance-minimization from demonstrations. Using multiple control and robotics environments, we show that MinSubFI frequently guarantees demonstrator acceptability, while existing imitation learning methods rarely do. 
Further, we show that MinSubFI with learned cost features provides demonstrator acceptability in our hand-specified cost feature space. 
Despite being designed for the more flexible setting in which the acceptability set can change for each demonstration, our experiments show that MinSubFI provides competitive true returns (without explicit assumptions about a static true cost function, as in other imitation learning methods). 
We additionally show that MinSubFI
is more robust to degradation in the quality of demonstrations used for training compared to existing approaches.

There are many interesting directions for future work. 
Learning feature representations without supplemental annotations is a challenging problem that would make our method easier to employ in practice.
Developing more general strategies for  snippet generation could better leverage demonstrations across different tasks or domains.
Exploring other methods for guaranteeing high levels of demonstrator acceptability is also of great interest. 
Finally, conducting experiments in application areas that lack true return functions for evaluation and/or have notions of acceptability that are dynamic and subjective is an important future direction.

\appendix


\section*{Acknowledgments}
This research is supported by NSF Award \#2312955.

\newpage 

\bibliographystyle{named}
\bibliography{bibliography}

\clearpage
\appendix

\section{Optimization of ${\alpha}$: Numerical and Analytical}\label{appendix:alpha_opt}
The $\alpha$ values of the subdominance define the sensitivity of the learned policy to not performing significantly better than demonstrations.
For a given imitator trajectory and set of demonstrations, the $\alpha$ values can be updated numerically via (exponentiated) stochastic gradient optimization (e.g., within Algorithm \ref{alg:the_alg}), as shown in Algorithm \ref{alg:alpha_update}.

\begin{algorithm}[H]
  \caption{Online update of $\alpha$ values}
  
  \label{alg:alpha_update}
  \begin{algorithmic}[1]
        \FOR{{\bf each } $k$}
            \STATE $\alpha_k\!\gets\!\alpha_k \exp\bigl\{-{\eta}_{t}' \sum\limits_{\mathclap{i,\tilde{\xi}_{i,j} \in \tilde{\Xi}_{i,m}^{\text{SV}_k}}} \bigl( f_{k}^{(\xi_{i})}-f_{k}^{(\tilde{\xi}_{i,j})}\bigr)\!+\!\lambda|\tilde{\Xi}|\alpha_k\bigr\}$
        
        \ENDFOR
  \end{algorithmic}
 \end{algorithm}

Alternatively, the optimal $\alpha$ values for can be computed analytically \cite{memarrast2023superhuman}:
\begin{align}
\alpha^*_k = \argmin_{\alpha_k} m \text{ such that: } f_k(\xi) + \lambda \leq \frac{1}{m} \sum_{j=1}^m f_k(\xi^{(j)}), .
\end{align}
where $\alpha_k^{(j)} = \frac{1}{f_k(\xi)-f_k(\tilde{\xi}^{(j)})}$ is the hinge slope that makes demonstration $\tilde{\xi}^{(j)}$ exactly where the subdominance becomes zero.

\section{Policy Gradient Subdominance Minimization}\label{app:minsubpg_thm_proof}
\begin{proof}[Proof of Theorem \ref{minsubpg_theorem}]
The general form of the gradient in the policy gradient update may be written as
\begin{align}\label{policy_grad_general_form}
    &g = \mathbb{E}_{\xi\sim\pi\times\tau}\left[\sum_{(s_t,a_t)\in\xi}G_t\nabla_{\theta}\log\pi_{\theta}(a_t|s_t)\right],
\end{align}
where $G_t$ measures the quality of acting under policy $\pi_\theta$ in state $s_t$ (e.g., 
discounted sum of future costs/rewards $\sum_{t+1}^{T}\gamma^{t}r(S_t)$, state-action value function $Q^{\pi_\theta}(S_t,A_t)$, advantage estimate $A^{\pi_\theta}(S_t,A_t)$, or a measure of expected future returns.

Further, the absolute subdominance of a trajectory can be decomposed over its states according to the equations of Corollary \ref{corollary_decompose}, which we rewrite for notational simplicity as:
\begin{align}
\textup{subdom}_{\boldsymbol{\alpha}}^{[\Sigma]}(\xi, \tilde{\Xi})
    & = 
    \sum_{s_t \in \xi}\textup{subdom}_{\boldsymbol{\alpha}}^{[\Sigma]}(s_t, \tilde{\Xi})\\
    & = \sum_{s_t \in \xi,k}\textup{subdom}_{\boldsymbol{\alpha}}^{[\Sigma],k}(s_t, \tilde{\Xi}),
\end{align}
where $\sum_{s_t \in \xi,k}\textup{subdom}_{\boldsymbol{\alpha}}^{[\Sigma],k}(s_t, \tilde{\Xi})$ is the \emph{contribution} of each state $s_t\in\xi$ towards the \emph{total} subdominance the trajectory $\xi$. It can be computed as:
\begin{align*}
    \textup{subdom}_{\boldsymbol{\alpha}}^{[\Sigma],k}(s_t, \tilde{\Xi}) = \!\!
        \frac{\tilde{C}^{k}}{|\xi|} + 
        \tilde{C}^{k}\alpha_k f_k(s_t) -
        \frac{\alpha_k \tilde{f}_{k,\textup{abs}}^{(j)}}{|\xi||\tilde{\Xi}|},
\end{align*}
where $\tilde{C}^{k}=\frac{|\tilde{\Xi}^{\textup{SV}_k}|}{|\tilde{\Xi}|}$, $\tilde{f}_{k,\textup{abs}}^{(j)}=\sum\limits_{\tilde{\xi}_{j} \in \tilde{\Xi}^{\textup{SV}_k}}\sum\limits_{s'_{t} \in \tilde{\xi}_{j}} f_k(s'_{t})$. 

Assume $G_t$ to be the total trajectory return in Equation \eqref{policy_grad_general_form}
$G_t=\sum_{s_t \in \xi}r(s_t)$,
and assume the subdominance \emph{contribution} of each state to be its \emph{negative} reward,
\begin{align*}
    r(s_t)=-\textup{subdom}_{\boldsymbol{\alpha}}^{[\Sigma]}(s_t, \tilde{\Xi})
\end{align*}
 Then Equation \eqref{policy_grad_general_form} may be rewritten as 
 \begin{align*}
     &g = \mathbb{E}_{\xi\sim\pi\times\tau}\biggl[ \sum_{(s_t,a_t)\in\xi} G_t\nabla_{\theta}\log\pi_{\theta}(a_t|s_t)\biggr]\\
     &=\mathbb{E}_{\xi}\biggl[ \sum_{(s_t,a_t)\in\xi} \sum_{s_t \in \xi}r(s_t)\nabla_{\theta}\log\pi_{\theta}(a_t|s_t)\biggr]\\
     &=\mathbb{E}_{\xi}\biggl[ \sum_{(s_t,a_t)\in\xi} \sum_{s_t \in \xi}-\textup{subdom}_{\boldsymbol{\alpha}}^{[\Sigma]}(s_t, \tilde{\Xi})\nabla_{\theta}\log\pi_{\theta}(a_t|s_t)\biggr]\\
     &=\mathbb{E}_{\xi}\biggl[\sum_{(s_t,a_t)\in\xi}-\textup{subdom}_{\boldsymbol{\alpha}}^{[\Sigma]}(\xi, \tilde{\Xi})\nabla_{\theta}\log\pi_{\theta}(a_t|s_t)\biggr]\\
     &=\mathbb{E}_{\xi}\biggl[-\textup{subdom}_{\boldsymbol{\alpha}}^{[\Sigma]}(\xi, \tilde{\Xi})\sum_{(s_t,a_t)\in\xi}\nabla_{\theta}\log\pi_{\theta}(a_t|s_t)\biggr],
 \end{align*}
where $\xi\sim\pi\times\tau$, and the final expression follows from the fact that the \emph{total} subdominance $\textup{subdom}_{\boldsymbol{\alpha}}^{[\Sigma]}(\xi, \tilde{\Xi})$ of trajectory $\xi$ is the same for each state $s_t\in\xi$. Substituting this gradient expression $g$ into the policy gradient update rule over multiple tasks $i$, we get the subdominance policy gradient update rule in Theorem \ref{minsubpg_theorem}. Alternatively, decomposing the relative subdominance according to 
its definition in Corollary \ref{corollary_decompose} gives us the equivalent result for the relative definition of subdominance.                                                                 
\end{proof}

\section{Per-State Cost Decomposition}\label{appendix:per_state_decomp}
The support vectors in subdominance minimization are the features of the demonstrations that the imitator trajectory fails to sufficiently dominate by a margin. 
When those support vectors are known, the trajectory subdominance can be decomposed over the individual states of a trajectory as follows.
\begin{proof}[Proof of Corollary \ref{corollary_decompose} (Absolute)]
The absolute, sum-aggregated subdominance is defined as:
{\small
\begin{align}
        &\textup{subdom}^{\Sigma}_{\boldsymbol{\alpha}}(\xi, \tilde{\Xi})
        =\sum_k \text{subdom}^k_{\alpha_k}(\xi,\tilde{\Xi})\notag\\
        &=\sum_k \frac{1}{|\tilde{\Xi}|}\sum_{\tilde{\xi} \in \tilde{\Xi}}\left[\alpha_k(f_k(\xi)-f_k(\tilde{\xi}))+1\right]_{+}\notag\\
        &=\sum_k \frac{1}{|\tilde{\Xi}|}\sum_{\tilde{\xi} \in \tilde{\Xi}^{\textup{SV}_k}(\xi)}\left(\alpha_k(f_k(\xi)-f_k(\tilde{\xi}))+1\right)
        \label{eq:decompProof}\\
        &=\sum_k \frac{\alpha_k}{|\tilde{\Xi}|}\sum_{{\tilde{\xi} \in \tilde{\Xi}^{\textup{SV}_k}(\xi)}}\left( \frac{1}{\alpha_k} + \sum_{\mathclap{s_t \in \xi}} f_k(s_t) - \sum_{\mathclap{s'_{t} \in \tilde{\xi}}} f_k(s'_{t}) \right)\notag
        \end{align}
        \begin{align}
        &= \sum_{s_t \in \xi,k} \frac{\alpha_k}{|\tilde{\Xi}|}\sum_{\tilde{\xi} \in \tilde{\Xi}^{\textup{SV}_k}(\xi)}\left(\frac{1}{\alpha_k|\xi|} + f_k(s_t)- \frac{\sum\limits_{\mathclap{s'_{t} \in \tilde{\xi}}}f_k(s'_{t})}{|\xi|}\right) \notag\\
        &= \sum_{\mathclap{s_t \in \xi,k}} \frac{\alpha_k|\tilde{\Xi}^{\textup{SV}_k}(\xi)|}{|\tilde{\Xi}|}\!\left(
        \frac{1}{\alpha_k|\xi|}\!+\! 
        f_k(s_t)\!-\!
        \frac{\sum\limits_{\tilde{\xi} \in \tilde{\Xi}^{\textup{SV}_k}(\xi)}\sum\limits_{s'_{t} \in \tilde{\xi}}f_k(s'_{t})}{|\xi||\tilde{\Xi}^{\textup{SV}_k}(\xi)|}\right) \notag\\
        &=\sum_{s_t \in \xi,k} 
        \frac{\tilde{C}_{\xi,\tilde{\Xi}}^{k}}{|\xi|} + 
        \tilde{C}_{\xi,\tilde{\Xi}}^{k}\alpha_k f_k(s_t) -
        \frac{\alpha_k \tilde{f}_{k,\xi,\tilde{\Xi}}^{\textup{abs}}}{|\xi||\tilde{\Xi}|}, \notag
\end{align}}%
\noindent
where $\tilde{C}_{\xi,\tilde{\Xi}}^{k}=\frac{|\tilde{\Xi}^{\textup{SV}_k}(\xi)|}{|\tilde{\Xi}|}$, and $\tilde{f}_{k,\xi,\tilde{\Xi}}^{\textup{abs}}=\sum\limits_{\tilde{\xi} \in \tilde{\Xi}^{\textup{SV}_k}(\xi)}\sum\limits_{s'_{t} \in \tilde{\xi}} f_k(s'_{t})$.

\end{proof}
Subdominance using the maximum over each feature to aggregate per-feature subdominances takes a similar form. 
It becomes identical to the expression starting from Equation \eqref{eq:decompProof} with the key difference that each demonstration can only be a support vector for a single feature dimension $k$ for max-aggregated subdominance (and conversely, each demonstration can be a support vector for multiple feature dimensions $k$ for sum-aggregated subdominance).

\begin{proof}[Proof of Corollary \ref{corollary_decompose} (Relative)]
The relative, sum-aggregated subdominance is defined as:
{\small
\begin{align}
        &\textup{relsubdom}^{\Sigma}_{\boldsymbol{\alpha}}(\xi, \tilde{\Xi})
        =\sum_k \text{relsubdom}^k_{\alpha_k}(\xi,\tilde{\Xi}_i) \notag\\
        &=\sum_k \frac{1}{|\tilde{\Xi}|}\sum_{\tilde{\xi} \in \tilde{\Xi}}\left[\alpha_k\left(\frac{f_k(\xi)}{f_k(\tilde{\xi})} - 1\right)+1\right]_{+} \notag\\
        &=\sum_k \frac{1}{|\tilde{\Xi}|}\sum_{\tilde{\xi} \in \tilde{\Xi}^{\textup{SV}_k}(\xi)}\left(\alpha_k\left(\frac{f_k(\xi)}{f_k(\tilde{\xi})} - 1\right)+1\right) \label{eq:decomp2}\\
        &=\sum_k \frac{1}{|\tilde{\Xi}|}\sum_{\tilde{\xi} \in \tilde{\Xi}^{\textup{SV}_k}(\xi)}\left( (\beta_k - \alpha_k) + \alpha_k \frac{\sum_{s_t \in \xi} f_k(s_t)}{\sum_{s'_{t} \in \tilde{\xi}} f_k(s'_{t})}  \right) \notag\\
        &= \sum_{s_t \in \xi,k}\!\Biggl( \frac{(1 - \alpha_k)|\tilde{\Xi}^{\textup{SV}_k}(\xi)|}{|\xi||\tilde{\Xi}|}\!+\!\frac{\alpha_k f_k(s_t)}{|\tilde{\Xi}|}\!\sum\limits_{\tilde{\xi} \in \tilde{\Xi}^{\textup{SV}_k}(\xi)}\frac{1}{\sum\limits_{\mathclap{s'_{t} \in \tilde{\xi}}} f_k(s'_{t})} \Biggr) \notag\\
        &= \sum_{s_t \in \xi,k} 
        \frac{\tilde{C}_{\xi,\tilde{\Xi}}^{k}(1 - \alpha_k)}{|\xi|} + \frac{\alpha_k f_k(s_t)\tilde{f}_{k,\xi,\tilde{\Xi}}^{\textup{rel}}}{|\tilde{\Xi}|},  \notag
\end{align}}%
\noindent
where $\tilde{C}_{\xi,\tilde{\Xi}}^{k}=\frac{|\tilde{\Xi}^{\textup{SV}_k}|}{|\tilde{\Xi}_i|}$, and $\tilde{f}_{k,\xi,\tilde{\Xi}}^{\textup{rel}}=\sum\limits_{\tilde{\xi} \in \tilde{\Xi}^{\textup{SV}_k}(\xi)} \left( \sum\limits_{s'_{t} \in \tilde{\xi}} f_k(s'_{t}) \right)^{-1}$.
\end{proof}

\section{Generalization Bound}\label{app:generalization_bound}

Our generalization bound relies on the absence of distinct local optima of the objective function. Formally, this is provided by the property of quasiconvexity, which guarantees that regions achieving a particular level of subdominance (or lower) are convex. 

\begin{lemma}
\label{lem:quasi}
When the set of features $\mathcal{F}$ realizable by the class of policies ($\mathcal{F} : {\bf f}(\xi) \in \mathcal{F}, \forall \xi \in \Pi$) is convex, the subdominance of realizable features is a quasiconvex function.
\begin{proof}
As a function of the realized features ($f_k$) of the imitator,
$\min_{\alpha_k} \text{subdom}^k_{\alpha_k,1}(f_k, \tilde{\Xi})$ is monotonic (increasing). 
Thus, $\min_{\alpha} \text{subdom}_{\alpha, {\bf 1}}({\bf f}, \tilde{\Xi})$ is a quasiconvex function of ${\bf f}$ for sum- or max-aggregated subdominance.
The intersection of any sublevel set of $\min_{\alpha} \text{subdom}_{\alpha, {\bf 1}}({\bf f}, \tilde{\Xi})$ with $\mathcal{F}$ is also convex. Therefore, 
$\min_{\bf f \in \mathcal{F}} \text{subdom}_{\alpha}({\bf f}, \tilde{\Xi})$ is a quasiconvex minimization problem.
\end{proof}
\end{lemma}

\begin{proof}[Proof of Theorem \ref{minsubpg_theorem_1}]
The generalization guarantee is based on leave-one-out cross validation error, which is an almost-unbiased estimate of generalization error under IID assumptions \cite{vapnik2000bounds}. 
Removing non-support vectors does not change global optima of subdominance minimization when no distinct local optima exist, which is the case for this quasiconvex optimization problem.
\end{proof}

\section{Implementation Details}
\subsection{Environments}
The environments used for our experiments are famous games re-implemented by OpenAI's gym \cite{brockman2016gym}, providing the tools and interface for interacting with reinforcement learning algorithms. We present the specifications and goals of the environments considered in this work. The available environments are separated into two categories: classic control (Cartpole, Lunar Lander) and MuJoCo environments (Hopper, Walker, HalfCheetah).
Observations in classic control environments are 1D state vectors.

\subsubsection{cartpole (\texttt{CartPole-v0})} The task is to keep a rotating pole, attached to a moving cart, vertical for as long as possible under a gravity model. The player can control the angle by moving the cart either left or right, each movement affecting the angular velocity of the pole. An episode terminates when the pole angle $\theta$ exceeds $\pm12^\circ$ (from the vertical y-axis) or, when the cart position $x$ exceeds $\pm2.4$. Maximum true return is 200.

\subsubsection{lunarlander (\texttt{LunarLander-v2})} The task is to land a shuttle that operates under a gravitational model, on the surface of the moon. An initial force is applied to the lander, providing with a starting velocity and angle; the player must then balance the shuttle and land it at the center of the screen, in a location delimited by two yellow flags. The player controls the shuttle by engaging one vertical and two lateral engines; the main vertical engine displaces the lander while the lateral ones pitch the lander.
In the MinSubFI implementation, we modified the lander to have fixed starting parameters (starting force and moon layout) which then characterize a single task; the maximum true reward in this case is approximately 310.

\subsubsection{hopper (\texttt{Hopper-v3})} The task is to control various joints of a hopping robot (restricted to the vertical $xz$-plane) to ``hop" and make forward progress. Reward at each timestep is a function of the forward velocity of the robot and its ``health" (determined by the physics engine based on the joint angles of the robot). 

\subsubsection{halfcheetah (\texttt{HalfCheetah-v3})} The task is to control various joints of a bipedal robot (restricted to the vertical $xz$-plane) to "run" and make forward progress. Reward at each timestep is a function of the forward velocity of the robot and its "health" (determined by the physics engine based on the joint angles of the robot).

\subsubsection{walker (\texttt{Walker2d-v3})} Adds an additional leg to the hopper environment so that the task is to ``walk" forward rather than ``hop".

\begin{table*}[h!]
  \caption{Description of cost features for each environment. $\sigma(x)=(1+\exp{(x)})^{-1}$ is the sigmoid function used to scale features to $[-1,+1].$}
  \label{tab:exp_cost_feats}
  \centering
  \begin{tabular}{cccl}
    \toprule
    Environment             & $\#$ Cost Features     & Cost Feature &  Description\\
    \midrule
    {\tt cartpole}    & 4  & $x^2$       & $(\text{cart position})^2$\\
                            &    & $v^2$       & $(\text{cart velocity})^2$\\
                            &    & $\theta^2$  & $(\text{pole angle})^2$\\
                            &    & $\omega^2$  & $(\text{pole angular velocity})^2$\\
    \midrule
    {\tt lunarlander} & 6  & $x^2$       & $(\text{lander x-position})^2$\\
                            &    & $y^2$       & $(\text{lander }y\text{-position})^2$\\
                            &    & $v_x^2$     & $(\text{lander }x\text{-velocity})^2$\\
                            &    & $v_y^2$     & $(\text{lander }y\text{-velocity})^2$\\
                            &    & $\theta^2$  & $(\text{lander angle})^2$\\
                            &    & $\omega^2$  & $(\text{lander angular velocity})^2$\\
                            & 3  & $\lVert a_t \rVert_{2}^{2}$ & control cost\\
    \midrule
    {\tt hopper}        & 2  & $-\sigma(v^{top}_x)+1$    & cost inversely proportional to $x$-velocity\\
                                   &    & $-\sigma(v^{top}_z)+1$ & cost inversely proportional to $z$-velocity\\
                        & 1  & $1-tanh(z)$    & cost inversely proportional to $z$-position\\
                        & 1  & $-\sigma(\theta)+1$    & cost inversely proportional to torso-angle\\
                        & 1  & $\sum\lVert a_t \rVert_{2}^{2}$ & control cost\\
    \midrule
    {\tt halfcheetah}   & 1  & $-\sigma(v^{top}_x)+1$    & cost inversely proportional to $x$-velocity\\
                        & 1  & $-v^{top}_x+10$           & cost inversely proportional to $x$-velocity\\
                        & 1  & $-[v^{top}_x]_{+}+10$   & cost inversely proportional to $x$-velocity\\
                        & 1  & $\sum\lVert a_t \rVert_{2}^{2}$ & control cost\\
    \midrule
    {\tt walker}   & 2  & $-\sigma(v^{top}_x)+1$    & cost inversely proportional to $x$-velocity\\
             &    & $-\sigma(-z)+1$    & cost inversely proportional to $z$-position\\
             & 1  & $\sum\lVert a_t \rVert_{2}^{2}$ & control cost\\
    \bottomrule
  \end{tabular}
\end{table*}

\subsection{Computing Cost Features}\label{app:comp_cost_feats}
For all environments employed in our experiments, we compute trajectory cost features $\textbf{f}:\Xi\rightarrow\mathbb{R}^{K}_{\geq0}$ that characterize the trajectory. Trajectory cost features are additive over the states of the trajectory, which are in turn characterized by state cost features $\textbf{f}:\mathcal{S}\rightarrow\mathbb{R}^{K}_{\geq0}$ or state-action cost features $\textbf{f}:\mathcal{S}\times \mathcal{A}\rightarrow\mathbb{R}^{K}_{\geq0}$.

The specific cost features employed differ between environments, but are chosen such that they may be readily computed from each environment's respective observation and/or action vector.

For example, in case of cartpole and lunarlander environments, for a given state $s_t$, the cost feature vector may be computed as 
\begin{align*}
    {\bf f}(s_t)=\{\phi_k(s_t)^2\}_{k=1}^K,
\end{align*}

where $\boldsymbol{\phi}(s):\mathcal{S}\rightarrow\mathbb{R}^{D}$ is simply the $D$-dimensional observation vector returned by the respective environments. Alternatively, $\boldsymbol{\phi}$ can be chosen to also be a function of the actions $a_t$,  $\boldsymbol{\phi}:\mathcal{S}\times\mathcal{A}\rightarrow\mathbb{R}^{D}$; this is useful in integrating control costs in the subdominance minimization problem. The cost feature set can be expanded to include any linear or non-linear, monotonic transformations of the entire cost feature vector ${\bf f}$ or a subset of its components ${\bf f}_k$. We can expand this cost feature set by computing the outer product of the original cost feature vector, ${\bf f}_{\text{expanded}}={\bf f}\cdot{\bf f}^{\top}$.

In essence, cost features are easily characterizable properties of each environment which, when minimized over a trajectory, allow an agent to successfully complete a task. Note, however, that these features are chosen carefully so as to not leak the true cost signal for an environment. For example, for the cartpole problem, the cost features comprise the pole angle $\theta^2$, pole angular velocity $\omega^2$, the cart position $x^2$, and the cart velocity $v^2$. While these cost features are pertinent to task-completion, they are unrelated to the true reward signal for the cartpole environment i.e., the number of timesteps elpased before the pole tips over.
The number of cost features defined is environment-specific; the complete list of cost features defined for each environment is provided in Table \ref{tab:exp_cost_feats}.

\begin{figure*}[thb]
    \centering
    \includegraphics[scale=0.50]{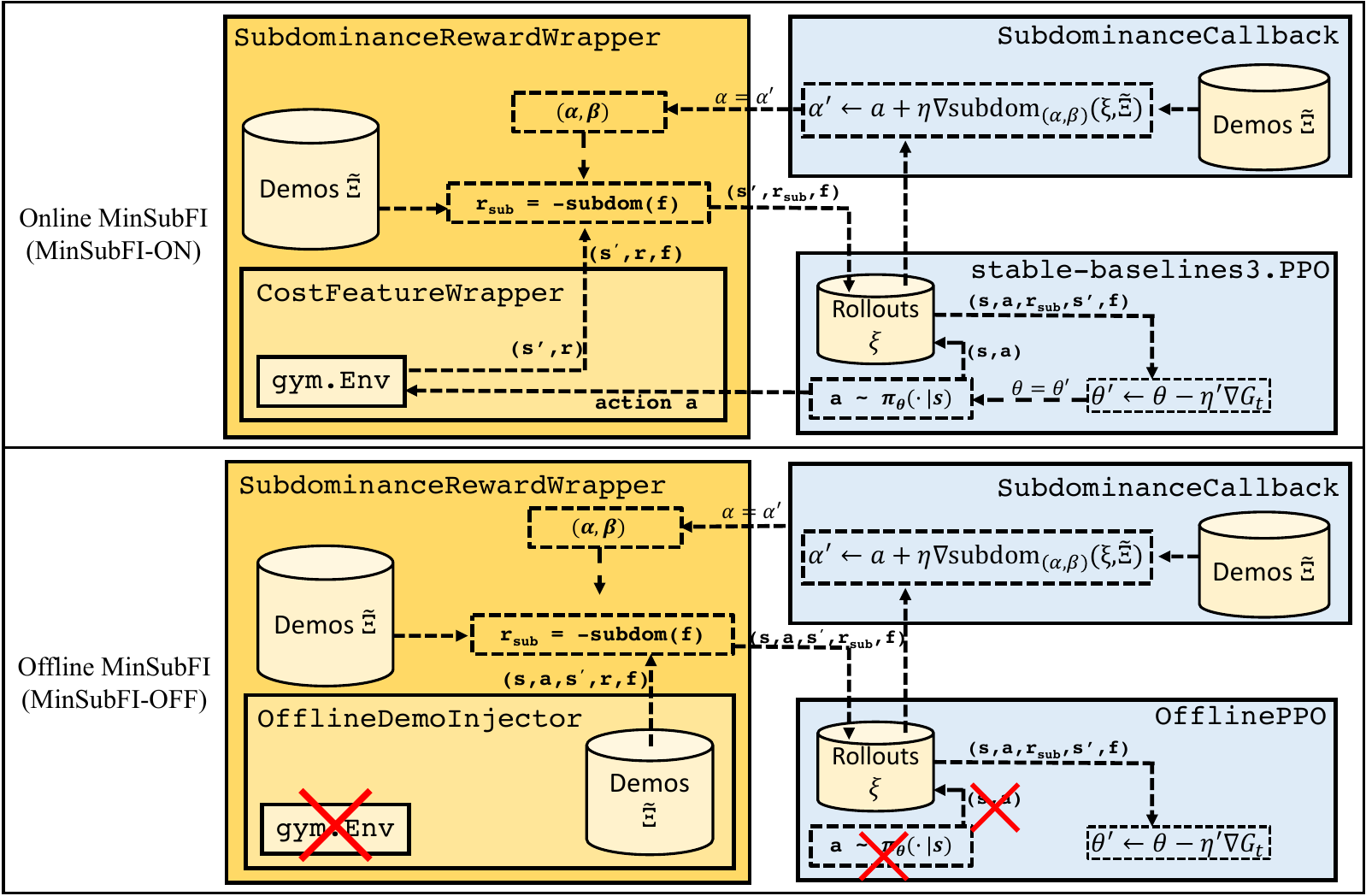}
    \caption{Implemented architecture of Online MinSubFI (top) and Offline MinSubFI (bottom) using $\texttt{gym}$ and $\texttt{stable-baselines3}$. The primary functionality of cost feature and subdominance computation is tackled via two environment wrappers $\texttt{CostFeatureWrapper}$ and $\texttt{SubdominanceRewardWrapper}$. The former computes cost features from observations and the latter computes subdominance relative to demonstrations using cost features. $\texttt{SubdominanceCallback}$ is called periodically to update $\alpha$. Offline MinSubFI uses the $\texttt{OfflineDemoInjector}$ wrapper around an environment to pass $(s,a,r,s',f)$ tuples from demonstrations as rollout data instead, and there is no action returned from the policy to the environment.} 
    \label{fig:sipgo_architecture}
\end{figure*}

\subsection{Trajectory Padding}
It follows from the definition of subdominance (Eq. \eqref{eq:subdominance}) that minimizing $f_k(\xi)$ naturally minimizes subdominance. Based on the specific definition of cost features employed, this sometimes results in degenerate policies. This degenerate behavior most commonly manifests as the trained policy learning to terminate episodes early to achieve lower subdominance via encountering fewer states in the trajectory. This phenomenon is best illustrated using the following example with a single, simple cost feature. 

\subsubsection{Example} Consider a problem setting where we employ a single cost feature $f$. An agent incurs cost features $f(s)=0$ upon reaching the terminal state $s_{\text{success}}$ and $f(s)=10$ in all other states (including $s_{\text{fail}}$.)
Now, consider three trajectories for this task -- a human demonstration ${\tilde \xi} = \{s_1, s_2, s_3, s_4, s_{\text{success}}\}$, and two trajectories $\xi_1 = \{s_1, s_2, s_3, s_4, s_5, s_6, s_{\text{success}}\}$ and $\xi_2 = \{s_1, s_2, s_{\text{fail}}\}$, sampled from policies $\pi_1$ and $\pi_2$ respectively. In choosing between the candidate policies $\pi_1$ and $\pi_2$, an agent opts for $\pi_2$, since, given any $\alpha$, $\pi_2$ results in lower subdominance despite not completing the task successfully. 
\begin{align*}
     f({\tilde \xi})=\sum_{s_t \in {\tilde \xi}}f(s_t) &= (4 \times 10) + 0 = 40\\
     f( \xi_1)=\sum_{s_t \in \xi_1}f(s_t) &= (6 \times 10) + 0 = 60\\
     f( \xi_2)=\sum_{s_t \in \xi_2}f(s_t) &= (3 \times 10) = 30\\
    \implies \text{[rel]subdom}^{[\Sigma]}_{\alpha,\beta} (\xi_1, \tilde{\xi}) &>       \text{[rel]subdom}^{[\Sigma]}_{\alpha,\beta} (\xi_1, \tilde{\xi}) \\
     \implies \xi_1 &\prec \xi_2.
\end{align*}

This toy examples gives us a peek into the source of this degeneracy. This phenomenon is very similar in nature to 'reward gaming' often encountered in other reinforcement learning settings \cite{armstrong2021pitfalls}. In our problem setting, this typically results from misalignment between the defined cost features and task objective, and is encountered experimentally when training is initialized from a random policy in such cases. For environments with such misaligned cost features and when starting subdominance minimization from a random policy, we employ the trajectory padding scheme described next.

\subsubsection{Padding Scheme} For an environment with misaligned cost features ${\bf f}^{(\text{mis})}\in\mathbb{R}^{K}_{\geq0}$, we fix a time horizon $h<H$ where $H$ is number of steps deemed sufficient to complete the task objective for that environment. The cost features of any short trajectory $\xi=({\bf f}^{(\text{mis})}_1,\dots,{\bf f}^{(\text{mis})}_T)$ where $T<h$ is padded with a fixed padding cost vector ${\bf f}_{\text{pad}}\in\mathbb{R}^{K}_{\geq0}$ up to the horizon $h$. The padded/augmented trajectory $\xi'$ then becomes:
\begin{align*}
    \xi'=({\bf f}^{(\text{mis})}_1,\dots,{\bf f}^{(\text{mis})}_T,{\bf f}^{(\text{pad})}_{T+1},\dots,{\bf f}^{(\text{pad})}_{h}).
\end{align*}
Intuitively, this enables a random policy to avoid degenerate solutions by augmenting the cost of such solutions.

\begin{table*}[tbh!]
  \caption{Hardware used for experiments.}
  \label{tab:hardware}
  \centering
  \small
  \begin{tabular}{ccccc}
    \toprule
    Machine Tag & OS & GPU (VRAM) & CPU & Memory \\
    \midrule
    Tabletop PC & Ubuntu 20.04  & GeForce RTX 3080 (10GB)         & AMD Ryzen 5 5600X            & 64 GB \\
    Lab Server  & Ubuntu 20.04 &  2 x  GeForce GTX 1080 Ti (12GB) & Intel Xeon E5-2697 & 180 GB \\
    Shared Cluster  & Ubuntu 20.04 &  2 x Tesla V100 (32 GB)         &  Intel Xeon Silver 4114 & 380 GB \\
    \midrule
    \bottomrule
  \end{tabular}
\end{table*}

\subsection{Reinforcement Learning}
For our experiments, we utilize a modified Proximal Policy Optimization approach for our policy gradient updates. We build on top of the implementation, provided by $\texttt{stable-baselines3}$ (SB3) \cite{stable-baselines3}; a library aimed towards offering robust implementations of important RL algorithms. We build our core subdominance minimization functionalities via wrapper classes for $\texttt{gym}$ environments and callback classes for $\texttt{stable-baselines3}$ model. The MinSubFI architectures for online and offline training are shown in Figure \ref{fig:sipgo_architecture}. Specifically, for online MinSubFI the $\texttt{CostFeatureWrapper}$ computes cost features $\texttt{f}$ for observation $\texttt{s}$ received from the environment. The $\texttt{SubdominanceRewardWrapper}$ contains the demonstrator's cost features which it uses to compute the subdominance; environment reward $\texttt{r}$ is replaced with negative subdominance $\texttt{r}_{\text{sub}}$ and returned to the $\texttt{PPO}$ agent. We find that equivalently returning the \emph{total} negative subdominance as a sparse cost in the terminal state of the rollout (i.e., avoiding the per-step cost decomposition from Corollary \ref{corollary_decompose}) works equally well in practice. The subdominance slopes $\alpha$  are updated via a periodic call to $\texttt{SubdominanceCallback}$. For offline MinSubFI, we create a dummy environment wrapper $\texttt{OfflineDemoInjector}$ which returns $(s,a,r)$ tuples sequentially from demonstrations {\emph concealed} as rollouts. We build $\texttt{OfflinePPO}$ and modify it's rollout collection to {\emph not} sample the policy, and instead treat the demonstrator's action as the one taken. Finally, the offline MinSubFI architecture in Figure \ref{fig:sipgo_architecture} shows three separate demonstration buffers only for sake of clarity.

\subsection{Snippet Subdominance Optimization}\label{appendix:snip_implementation}

When sampling imitator trajectories from $s_{t_s}$ (Line 5 in Algorithm \ref{alg:the_alg_snip_opt}), we follow the policy $\pi_{\boldsymbol{\theta}}(\cdot|s_{t_s})$ for a fixed $T$ steps (rather than until episode termination). For different environments, we choose $T$ to be between $10$-$25\%$ of the maximum trajectory length for the environment. The demonstration trajectory is similarly chosen to be $T$ steps following $s_{t_s}$. We then divide the $T$-step imitator and demonstrator trajectories in $N$ equal, non-overlapping snippets, each $T/N$ steps long. We then compute the subdominances of all $N^2$ pairwise combinations of imitator and demonstrator snippets and consider the imitator snippet with the smallest subdominance for each demonstrator snippet. Note all successive snippets considered from a trajectory all start at initial state $s_{t_s}$ ($t=0$) and end at timesteps $t=T/N$, $2T/N$, and so on. Fixing the imitator and demonstrator's trajectories to be $T$ steps long, and divided into $N$ snippets ensures that snippets from these trajectories are of comparable length.

\subsection{Hardware Details 
}
\label{sec:hyperparams}
Details of the hardware employed are provided in Table \ref{tab:hardware}. Training the online version of MinSubFI end-to-end requires approximately 2 hours for $9e6$ environment interactions, on lunarlander, utilizing approximately 15\% of the GPU's memory; measured on a tabletop computer equipped with an NVIDIA GeForce RTX 3080 GPU and Ryzen 5600X CPU. The other two more powerful machines were used for concurrent experimentation. 

\begin{figure}[ht]
    \centering
    \includegraphics[width=\linewidth]{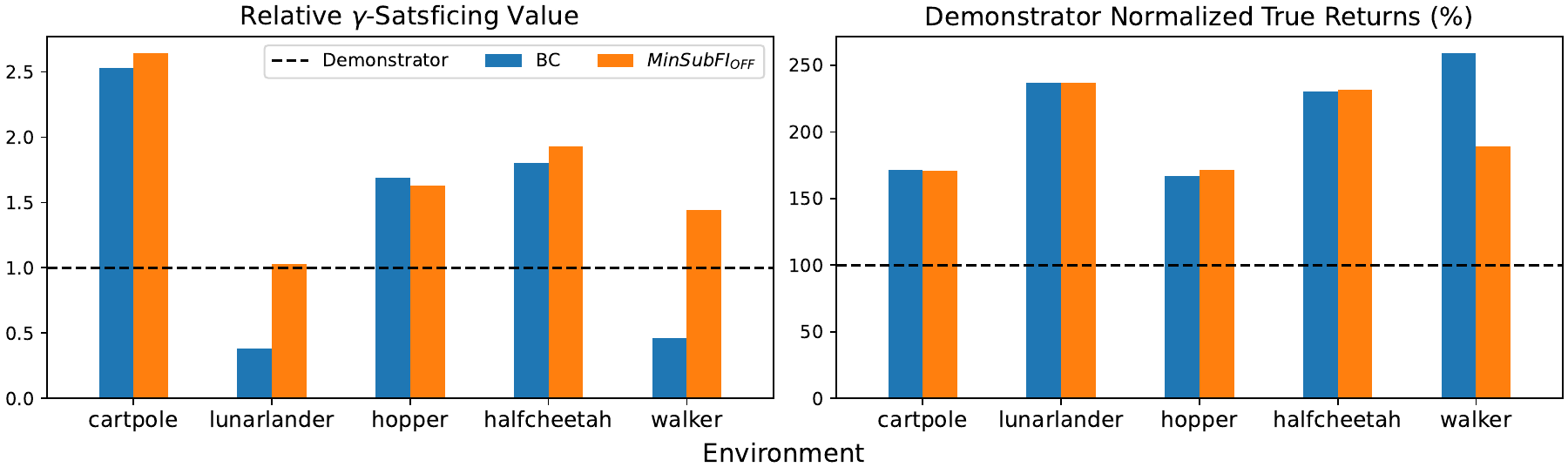}
    \caption{Impact of behavior cloning (BC) and offline MinSubFI (MinSubFI\textsubscript{OFF}) policy initialization on the demonstrator satisfaction (left) and demonstrator-normalized true returns (right) of our online MinSubFI algorithm (averaged over the same 5 seeds, demonstrator performance indicated by black dashed line). The comparable true returns and higher rates of demonstrator satisfaction suggest a MinSubFI\textsubscript{OFF} policy initialization rather than BC.}
    \label{fig:policy_init_ablation}
\end{figure}

\section{Ablation Study: Policy Initialization}
\label{sec:policy_init_ablation}
To better understand the impact of policy initialization, we compare the online version of our algorithm (MinSubFI\textsubscript{ON}) when trained from one of two pretrained policies: a BC policy and a policy trained using the offline version of our algorithm (MinSubFI\textsubscript{OFF}). The results in this study are averaged over the \textit{same} 5 randomly-chosen seeds; using the same 5 seeds for both initializations ensures that the same demonstrations are chosen for training and evaluation (for any given seed). Figure \ref{fig:policy_init_ablation} shows the relative $\gamma$-satisficing values (Figure \ref{fig:policy_init_ablation}, left) and demonstrator-normalized true returns (Figure \ref{fig:policy_init_ablation}, right) resulting from both initializations. While the true returns between both initializations do not differ significantly, the offline initialization consistently guarantees a higher rate of demonstrator satisfaction. Consequently, we use a MinSubFI\textsubscript{OFF} policy initialization throughout our experiments.

\end{document}